\documentclass[twoside,11pt]{article}
\usepackage[preprint]{jmlr2e}

\usepackage{amsmath}  % in preamble

\usepackage{booktabs}       % professional-quality tables
\usepackage{hyperref}
\usepackage{enumitem}
\usepackage{caption}
\usepackage{xcolor}
\usepackage{wrapfig}
\usepackage{comment}
\usepackage[ruled]{algorithm2e}
\newtheorem{property}{Property}[section]

\usepackage{lastpage}
\firstpageno{1}

\begin{document}
%opening
\title{Transformers know more than they can tell \\ Learning the Collatz sequence}

\author{
    \name Fran\c{c}ois Charton \email francois@axiommath.ai\\ 
    \addr Axiom Math\\ 
    CERMICS, Ecole Nationale des Ponts et Chaussées 
    \AND 
    \name Ashvni Narayanan \email ashvni.narayanan@sydney.edu.au\\     \addr Sydney Mathematical Research Institute, University of Sydney
}

 \editor{}

\maketitle

\begin{abstract}
We investigate transformer prediction of long Collatz steps, a complex arithmetic function that maps odd integers to their distant successors in the Collatz sequence ( $u_{n+1}=u_n/2$ if $u_n$ is even, $u_{n+1}=(3u_n+1)/2$ if $u_n$ is odd). Model accuracy varies with the base used to encode input and output. It can be as high as $99.7\%$ for bases $24$ and $32$, and as low as $37$ and $25\%$ for bases $11$ and $3$. Yet, all models, no matter the base, follow a common \emph{learning pattern}. As training proceeds, they learn a sequence of classes of inputs that share the same residual modulo $2^p$. Models achieve near-perfect accuracy on these classes, and less than $1\%$ for all other inputs. 
This maps to a mathematical property of Collatz sequences: the length of the loops involved in the computation of a long Collatz step can be deduced from the binary representation of its input. The learning pattern reflects the model learning to predict inputs associated with increasing loop lengths. 
An analysis of failure cases reveals that almost all model errors follow predictable patterns. Hallucination, a common feature of large language models, almost never happens. In over $90\%$ of failures, the model performs the correct calculation, but wrongly estimates loop lengths. 
Our observations give a full account of the algorithms learned by the models. They suggest that the difficulty of learning such complex arithmetic function lies in figuring the control structure of the computation -- the length of the loops. We believe that the approach outlined here, using mathematical problems as tools for understanding, explaining, and perhaps improving language models, 
can be applied to a broad range of problems and bear fruitful results.
\end{abstract}

\section{Introduction}

Transformers~\citep{vaswani17}, the prevalent architecture in deep learning models~\citep{radford2019,touvron2023,gemini2025}, prove surprisingly weak when trained on simple arithmetic operations~\citep{lee2023,hanna2023,wang2021}. Learning long integer addition, a basic operation, requires advanced techniques, such as scratchpads~\citep{nye2021} or chain-of-thought~\citep{wei2023}. Generalizing to larger (or smaller) operands that those used during training requires problem-specific positional embeddings~\citep{mcleish2024, zhou2023}. In many recent models, integer multiplication fails unless one of the operands is small~\citep{qiu2024,dziri2023}, and complex arithmetic operations cannot be learned~\citep{lee2023}.
These limitations are all the more shocking as transformers have been used to solve hard problems of mathematics, such as symbolic integration~\citep{lample2019}, discovering Lyapunov functions and Gr\"obner bases~\citep{alfarano2024, kera2024,kambe2025}, or finding new solutions to old problems in combinatorics~\citep{charton2024patternboost,ellenberg2025}. 

Workarounds have been proposed, such as interfacing with external tools~\citep{schick2023} or generating computer code that performs the calculations~\citep{Funsearch}. Theoretical papers have also investigated the limitations of transformer architectures when performing exact computations on very long sequences~\citep{zhou2023,weiss2021}. Still, little is understood about why transformers struggle with arithmetic tasks, and how such functions are learned, from examples only. Prior works  focus on small models and simple operations, such as modular addition~\citep{zhong2023,nanda2023}, the M\"obius function~\citep{lowryduda2025}, or the greatest common divisor~\citep{charton2024gcd}.

In this paper, we train small transformers to learn a complex arithmetic function, the long Collatz step (Sec.~\ref{sec:long_collatz}): predicting a distant successor in the Collatz sequence ($u_{n+1}=(3u_n+1)/2$ if $u_n$ is odd, $u_{n+1}=u_n/2$ if $u_n$ is even). We chose this particular function for a number of reasons. First, the Collatz sequence, despite having a very simple definition, is known to give rise to complex mathematical behavior. Because it is the subject of a longstanding conjecture, it was heavily researched. We will leverage this prior knowledge when designing and interpreting our experiments. Second, the algorithm for computing the long Collatz step features two loops of variable length. This is a much more complex task than the basic arithmetic operations studied in prior works. Finally, as its input varies, the long Collatz step takes a lot of different values ($23.5$ million different values over the $50$ million first odd integers). Prior works focused on functions with a small number of outcomes. Learning such functions amount to a pattern recognition problem: mapping inputs onto a small number of classes. Learning the long Collatz step, on the other hand, amounts to a \emph{regression} problem, a harder setting for learning machines.

We compare models using different bases to represent integers (Sec.~\ref{sec:base}). The best models (bases $24$, $16$ and $32$) achieve a near-perfect accuracy of $99.7\%$, while odd-base models struggle to get past $80\%$. Yet, all models learn the long Collatz step in a sequence of discrete steps, where specific classes of input are learned and predicted with very high accuracy (Sec.~\ref{sec:pattern}). This \emph{learning pattern} is independent of the base, and corresponds to a deep mathematical property of the Collatz sequence: the structure of a long Collatz step, the length of the loops involved in its computation, can be deduced from the binary representation of its input (Sec.~\ref{sec:theory}).

An investigation of model errors (Section~\ref{sec:errors}) reveals that, whereas large language models commonly ``hallucinate'' random solutions, our models fail in principled ways. In almost all cases, the models perform the correct calculations for the long Collatz step, but use the wrong loop lengths, by setting them to the longest loop lengths they have learned so far. This unites model errors and the learning pattern, and provides a full account of the algorithm learned by our models. 
It also suggests that the hardness of the task does not lie in learning the arithmetic operations but in figuring the control structures of the algorithm, the lengths of the loops, which can be deduced from the binary representation of the input.

These results shed light on the learning patterns at work in math transformers. They confirm that, when learning the long Collatz steps, models do not rely on shortcuts or shallow statistical properties of the task, but learn deep mathematical properties of the problem, which allow< them to perform better than their accuracy figure suggest, that is, fail in principled ways instead of hallucinating.
We believe that the method we propose to investigate the long Collatz step can be applied to many problems, and help understand how transformers learn mathematics.

\section{Related works}

Neural networks were first applied to problems of arithmetic in the early 1990s~\citep{SiuRoychowdury92}. Recurrent architectures (precursors of the transformer) were proposed by~\citet{Zaremba15}, \citet{kalchbrenner15}, and~\citet{kaiser2015neural}. Recent research mostly focuses on fine-tuning large pre-trained transformers on arithmetic tasks, to solve math word problems~\citep{meng2019,griffith2021}. 
The hardness of arithmetic was first observed by~\citet{saxton2019}. For a survey of the limitations of transformers on arithmetic tasks, see~\citet{lee2023} and~\citet{dziri2023}.

Long integer addition received most of the attention. Scratchpad methods, that allow transformers to add long integer, were proposed by~\citet{nye2021}. The importance of base representation was discussed by~\citet{nogueira2021}.
Length generalization, adding longer (or shorter) integers than those in the training set was discussed by~\citet{jelassi2023}, and addressed by~\citet{mcleish2024}, via a task-specific positional embedding.

Modular and finite field arithmetic was investigated in the context of grokking~\citep{power2022}, delayed generalization of overfitted models~\citep{gromov2023,nanda2023}, and cryptography~\citep{wenger}. In \citet{saxena2025} and ~\citet{charton2024repeated}, we propose losses, embeddings and training techniques to help learn such tasks.

In \citet{charton2024gcd}, we train transformers to predict the greatest common divisor (GCD) of two integers, a basic arithmetic operation with a small number of possible outcomes. The methodology is very close to ours: small transformers trained on randomly generated examples, using different bases to encode integers. Models learn to cluster their input pairs into classes sharing a common divisor, a learning pattern similar to ours. However, the pattern depends on the encoding base. Adjusting the training distribution, by oversampling small operands and large outcomes, brings a large improvement in model performance, a phenomenon we do not observe for the long Collatz steps (Section~\ref{sec:uniform}).

Most prior work on math transformer interpretability focuses on model weights~\citep{nanda2023,zhong2023}. Methods similar to ours were proposed for linear algebra~\citep{charton2022}, and the GCD~\citep{charton2024gcd}.

The Collatz conjecture (all Collatz sequences eventually reach $1$), also known as the 3x+1 problem, Ulam conjecture, Kakutani’s problem, Thwaites conjecture, or the Syracuse problem, was proposed by Collatz in 1937, when he was a student. It is often cited as an example of a very hard problem with a very simple statement, and was heavily studied \citep[for a summary]{lagarias2021}. It holds true for all starting numbers up to $2^{71}$.

Many generalizations of the Collatz conjecture have been proposed: to positive reals ~\citep{konstadinidis2006}, with different constants~\citep{bruschi2008}, over $\mathbb{F}_2$~\citep{gallardo2023}, in polynomial rings over commutative rings~\citep{behajaina2024}, to rationals with odd numerators and denominators ~\citep{kobus2025}. It is related to other hard problems, such as the ABC conjecture~\citep{Rozier2025}, Terras' conjecture~\citep{rozierpar2025}, and the existence of special additive systems in $\mathbb{Z}$~\citep{zabolotskii2025}. 

Recent attempts at proving the Collatz conjecture used dynamical systems~\citep{siegel2024,fu2025,dominici2007}, operator theory ~\citep{mori2025} and graph theory ~\citep{Andrei2000}. \citet{tao2022}, using a probabilistic argument, proved that almost all Collatz orbits attain almost bounded values.

\section{Long Collatz steps}\label{sec:long_collatz}

The Collatz sequence $\mathcal C(n)$ associated with a positive integer $n$ is defined as: 
\begin{align*}
    c_0 &= n, \\
    c_{i+1} &=  \frac{c_i} 2, \text{if $c_i$ is even (a \emph{down-step}),} \\
    c_{i+1}  &=  \frac{3 c_i + 1} 2, \text{if $c_i$ is odd (an \emph{up-step}).} 
\end{align*}

$\mathcal C(1)$, the Collatz sequence associated to $1$ is the infinitely repeating loop $(1,2,1)$. It is conjectured that any Collatz sequence $\mathcal C(n)$ eventually reaches $1$. 

Assume $n$ is odd, and let $2^k$ be the largest power of two dividing $n+1$, that is, $n=2^km-1$ with $m$ odd and $k \geq 1$. Since $c_0=n$ is odd, $c_1= \frac{3 n +1} 2  =   2^{k-1} (3m) - 1,$ an odd integer if $k>1$.  Repeating this $k$ times, we have.

\begin{property}Any odd positive integer of the form $n=2^k m-1$, with $k\geq 1$ and $m$ odd, is transformed, after $k$ Collatz up-steps, into $c_{k}=3^km-1,$ an even positive integer.\end{property}

Let $k'\geq 1$ be the largest power of two dividing $c_{k}$, i.e. $c_{k}=2^{k'}p$ with $p$ odd. The next $k'$ steps of the Collatz sequences will be division by $2$ (down-steps), and we have $$c_{k+k'}=\frac{3^km-1}{2^{k'}}=\frac{{(\frac3 2)}^k (n+1) -1}{2^{k'}}.$$ 

We call this sequence of $k+k'$ Collatz steps a \textbf{long Collatz step}, and define, for any odd $n=2^km-1$, its \textbf{long Collatz successor}\footnote{We could define the long Collatz successor of an even integer as $\kappa(2^km)=m.$} as $$\kappa(n)=\frac{{(\frac3 2)}^k (n+1) -1}{2^{k'}}.$$
The long Collatz step transforms an odd integer into an odd integer. It has $1$ as a fixed point, the only one if the Collatz conjecture is true. 

To compute $\kappa(n)$ for a given $n$, one first applies $k$ times the transformation $n \to \frac {3n+1}2$, then $k'$ times $n \to \frac n 2$. $k$ and $k'$ are functions of $n$. $k$ is the largest power of two that divides $n+1$, i.e. the number of ones at the right of the binary representation of $n$. $k'$ is the largest power of two dividing $c_{k}=3^km-1= {(\frac 3 2)}^k (n+1)-1$, the \emph{apex} $a$ of the long Collatz step. 

Therefore, the algorithm for computing the long Collatz step involves two loops of variable length. If we know the binary representation of $n$ and $a$, or can test the parity of intermediary results, we can compute $\kappa(n)$ with one of the two algorithms below.

\begin{algorithm}[]
    \small
    $k \gets$ number of right $1$s in the binary representation of $n$\;
    \For { $1\leq i \leq k$}{
     $n\gets \frac{3n+1}{2}$
    }
    $a \gets n$\;
    $k' \gets$ number of right $0$s in the binary representation of $a$\;
    \For {$1\leq i \leq k'$}{
     $n\gets \frac{n}{2}$
    }
    \Return $n$
    \caption{\small $\kappa(n)$ with two \texttt{for} loops}\label{alg:kappa_with_for}
\end{algorithm}

\begin{algorithm}[]
    \small
    \While { $n \equiv 1 \pmod 2$}{
     $n\gets \frac{3n+1}{2}$
    }
    \While {$n \equiv 0 \pmod 2$}{
     $n\gets \frac{n}{2}$
    }
    \Return $n$
    \caption{\small $\kappa(n)$ with two \texttt{while} loops}\label{alg:kappa_with_while}
\end{algorithm}

\section{Learning the long Collatz step}\label{sec:base}

We train transformers to predict $\kappa(n)$, the long Collatz successor of $n$, from a large set of pairs $(n,\kappa(n))$, with $n$ odd, uniformly drawn between $1$ and $10^{12}$. Our models are sequence-to-sequence transformers~\citep{vaswani17}, with a bidirectional encoder, with $4$ layers, an embedding dimension of $512$, and $8$ attention heads, and an autoregressive decoder with $1$ layer, a dimension of $512$ and $8$ heads, connected to the encoder by a cross-attention mechanism.

Model inputs and outputs are encoded as sequences of digits in base $B$. We train $56$ models, for all bases from $2$ to $57$, on about $300$ million random pairs $(n, \kappa(n))$. Models minimize a cross-entropy loss, using the Adam optimizer with a learning rate of $3\cdot 10^{-5}$. 

At regular intervals (every 300,000 training examples, an ``epoch'' henceforth), model accuracies are measured on a test set of 100,000 random pairs, as the percentage of exact predictions of $\kappa(n)$. The large size of the problem space ($5\cdot 10^{11}$ possible values of $n$) guarantees that there is almost no overlap between training and test set. 

Table~\ref{tab:base_acc} summarize trained model accuracies. The best performing models, using base $24$, $16$ and $32$, achieve $99.7\%$ accuracy, a surprising result given the complexity of the task and the low performance of arithmetic transformers reported in prior works~\citep{lee2023,charton2024gcd}. Performance varies with the base: all even-base models achieve $88\%$ accuracy or more. Most odd-base models are below $80\%$ accuracy. Accuracies tend to concentrate around a small set of values: of the $28$ odd base models, $12$ achieve $51\%$ accuracy, $8$ achieve $55\%$ and $4$ achieve $81\%$.

\begin{figure*}[!t]
\begin{minipage}{0.5\textwidth}
    \small
    \centering
    \begin{tabular}{ll}
        \toprule
         Accuracy & Bases  \\
        \midrule
        99.5+\% & 24, 16, 32, 36, 12, 48, 8 \\
        99 -- 99.5\% & 4, 18, 56, 40  \\
        95 -- 99\% & 6, 20, 28, 2, 52, 10, 44, 54 \\
        91 -- 95\% & 42, 22, 30, 14, 50, 34 \\
        88 -- 89\% & 26, 46, 33\\
        81 -- 82\% & 9, 45, 38, 15, 17 \\
        70 -- 71\% & 49, 47, 21, 27, 51, 39, 31, 41, 23, 7, 25 \\
        59 -- 66\% & 57 (65\%), 43 (59\%)\\ 
        55 -- 56\% & 13, 37, 19, 55, 29, 53, 35, 5 \\
        25 -- 37\% & 11 (37\%), 3 (25\%) \\
       \bottomrule
    \end{tabular}
    \captionof{table}{\small Model accuracy for different bases, after training on 300 million examples. Bases listed by decreasing accuracy.}
    \label{tab:base_acc}
\end{minipage}
\hfill
\begin{minipage}{0.44\textwidth}
\includegraphics[width=0.90\textwidth]{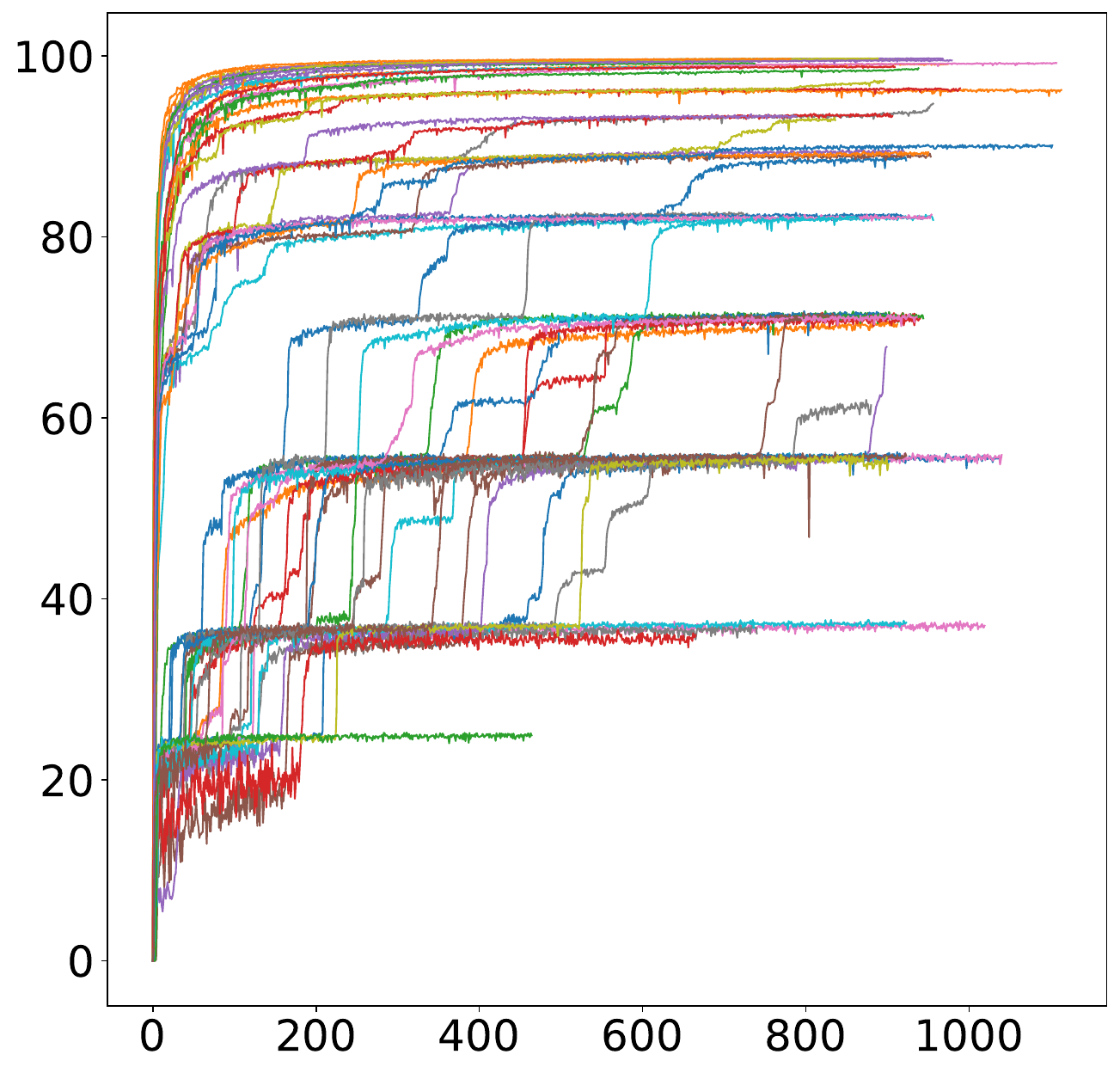}
\captionof{figure}{\small  Learning curves for bases 2 to 57 (accuracy vs epochs of 300,000 examples).}
    \label{fig:stepcurves}
\end{minipage}\\
\end{figure*}

The learning curves (Figure~\ref{fig:stepcurves}) confirm this phenomenon. 
Learning happens in steps, as model accuracies jump from one level to the next. These accuracy levels are \emph{quantized}: they take the same values for all bases.
As training proceeds, all odd-base model accuracies step from $25\%$ to $37$, $55$, $71$ and sometimes $81\%$. Even-base model step from $71\%$  accuracy to $81$, $88$ and $92\%$.
This \emph{learning pattern} is intriguing. Step-like learn curves were observed in prior works on arithmetic \citep{charton2024gcd}, but accuracy levels varied with the base. The fact that all bases share the same accuracy levels suggests that the learning pattern is a feature of the Collatz sequence.

\subsection{The learning pattern -- Correct predictions in different bases}\label{sec:pattern}

Models using different bases not only achieve the same accuracy levels, they also succeed and fail on the same inputs. Consider the base-$15$ and base-$17$ models, which both achieve $81\%$ accuracy. If we test them on the \emph{same} sample of 100,000 examples, we notice that they agree in $98\%$ of cases: they are both correct on 80,997 examples, and both wrong on 16,955. The two models disagree (one is correct, the other is wrong) on 2,048 examples, $2\%$ of the test set or about $10\%$ of model errors. This suggests that the steps in the learning pattern are associated with particular classes of inputs, that all models learn.

To characterize these classes, we calculate model accuracy on test examples that share a common residual modulo $2^p$. For $p=3$, we compare model performance on inputs of the form $8k+1$, $8k+3$, $8k+5$ and $8k+7$. 

We notice that, for $p$ large enough, all residual classes modulo $2^p$ are either predicted with almost perfect accuracy (over $95\%$), or not learned at all (less than $1\%$). For instance, the base-$11$ model achieves $99.9, 99.9$ and $96.6\%$ accuracy on inputs congruent to $1$, $9$ and $11$ modulo $16$, and less than $1\%$ on all other inputs ($0.9\%$ for inputs congruent to $7$, less than $0.1\%$ on other classes). This accounts for the $37\%$ accuracy reached by the base-$11$ model. The nine models that achieve $55\%$ accuracy predict inputs congruent to $1$ modulo $8$, $11$ and $13$ modulo $16$ and $7$ modulo $32$ with over $95\%$ accuracy, and all others with less than $1\%$. This pattern is observed in all bases, throughout training. Accuracy jumps up when a new residual class is learned, and the order of learning is the same for all bases.

Residual classes modulo $2^p$ correspond to specific endings in the binary representation of inputs. Inputs congruent to $1$ modulo $8$ are the positive integers that have a binary representation ending in $001_2$. Therefore, our base-$11$ model has learned to predict inputs ending in $001_2$ and $1011_2$ (and none others). Models with $55\%$ accuracy have learned two additional classes: inputs ending in $1101_2$ and $00111_2$, and models achieving $71\%$ accuracy (e.g. base $21$) predict three more: $00101_2$, $10011_2$ and $101111_2$. The same pattern holds in even-base models: the base-$50$ model, which achieves $92.3\%$ accuracy, predicts inputs ending in $01_2$, $011_2$, $00111_2$, $101111_2$, $110111_2$, $1001111_2$, $1010111_2$, $00010111_2$ and $1110001111_2$. 

This provides a complete account of the learning pattern. Every step in the learning curve happens when the model learns to predict (with near-perfect accuracy) a new class of inputs, characterized by their binary ending. All bases learn these classes in the same order, beginning with  $001_2$, $1011_2$, $1101_2$, $00111_2$ and $10011_2$. Model accuracy is the frequency of these classes in the test set: twice the sum of the moduli, e.g. $2(\frac{1}{8}+\frac{1}{16})$ for base $11$. The learning pattern applies to all bases, but learning is faster in even bases, because the binary representation of inputs is easier to learn (all even numbers end in particular digits).

\subsection{Why the binary classes -- Learning k and k'}\label{sec:theory}
The learning pattern, common to all bases, appears to be a property of the long Collatz step. In fact, we will now show that the residual classes correspond to inputs $n$ associated with particular values of $k$ and $k'$, the lengths of the two loops in the computation of $\kappa(n)$. In other words, instead of learning the mapping $$\kappa(n) = \frac{(\frac 3 2)^k (n+1)-1}{2^{k'}},$$ with $k$ and $k'$ functions of $n$, the models learn its restriction to specific values of $k$ and $k'$.

Let $n$ be an odd positive integer, and let $B_n$ be its binary representation. We know from Section~\ref{sec:long_collatz} that $k$, the number of iterations in the first loop of the long Collatz step, is the number of $1$s at the right of $B_n$. Thus, for any $n$ with $k=1$, $B_n$ must end in $01$, if $k=2$, $B_n$ must end in $011$, and so on.

Suppose $k=1$, then $B_n$ ends in $01$ and $n=c_0=4m+1$. Let us compute its long Collatz successor. After one iteration of the first loop ($n \to \frac{3n+1}2$), we reach $c_1=\frac{3n+1}2=6m+2$, an even number. This is the apex: all future steps will be divisions by $2$, $k'\geq 1$ is the largest power of two dividing $c_1$, and we have $$\kappa(n)=c_{1+k'}=\frac{6m+2}{2^{k'}}.$$ 
Dividing $c_1$ by two yields $c_2=3m+1$. We now have two cases.

\emph{(1)} If $m$ is even, $c_2=3m+1$ is odd. Then $c_2=\kappa(n)$ and $k'=1$. Rewriting $m=2p$, we have $n=8p+1$. Therefore, $B_n$ ends in $001$. This is the first class in the learning pattern, which contains all odd integers for which $k=k'=1$. 

\emph{(2)} If $m$ is odd, we can write $m=2p+1$ and $B_n$ ends in $101$. Then $c_2=6p+4$, an even integer, so $k'\geq 2$, and $c_3=3p+2$. At this point, we again have two cases.

\emph{(2a)} If $p$ is odd, $c_3=3p+2$ is odd, so $c_3=\kappa(n)$ and $k'=2$. $B_n$ ends in $1101$. This is the third class in the learning pattern: inputs for which $k=1$ and $k'=2$.

\emph{(2b)} If $p$ is even, $B_n$ ends in $0101$, $c_3=3p+2=6q+2$ is even, so $k'\geq 3$, and $c_4=3q+1$. We once again have two cases but since $c_4 = 3q+1$, they are the same as the ones we discussed for $c_2$, and the same reasoning applies.

Iterating, we see that, for inputs with $k=1$, i.e. $B_n$ ending in $01$, we have $k'=1$ iff $B_n$ ends in $001$, $k'=2$ iff $B_n$ ends in $1101$, $k'=3$ iff $B_n$ ends in $00101_2$, and $k'\geq 4$ if  $B_n$ ends in $10101_2$. By induction, we can prove that any $k'$ can be deduced from $B_n$, with the following rules:

\begin{itemize}[nosep]
\item $k'=2p+1$ if $B_n$ ends in $0(01)^p01$ for some $p\geq 0$,
\item $k'=2p+2$ if $B_n$ ends in $11(01)^p01$ for some $p \geq 0$.
\end{itemize}

We can reformulate this as a general property of integers $n \equiv 1 \mod 4$.
\begin{property}
Let $n \equiv 1 \mod 4$, its binary representation $B_n$ ends in $01$. Consider the binary sequence $S_1=(01)^p01$, with $p$ chosen large enough that $S_1$ is longer than $B_n$, and let $q\geq 2$ be the longest matching suffix between $B_n$ and $S_1$, then $k=1$ and $k'=q-1$. 
\end{property}

This property is an instance of a more general theorem, valid for all $k$, that we will state and prove. Let $\Phi(n)$ stand for the Euler function, $|x|$ for the length of the binary sequence $x$, and let addition of binary sequences stand for concatenation. We first prove the lemma.

\begin{lemma}{\label{lemm:a_p}}
For any $l>0$, the sequence $a_{l,p}=2^{-p} \cdot (3^l-1)\mod 3^l$ is periodic, with period  $\Phi(3^l)=3^l-3^{l-1}$.
\end{lemma}

\begin{proof}
Since $3$ is prime, for any $l\geq 1$, the multiplicative group $\mathcal G=(\mathbb Z / 3^l \mathbb Z)^\times$ is cyclic, and has order $\Phi(3^l)$. In fact, $\mathcal G$ includes all the positive integers up to $3^l-1$, co-prime with $3^l$. Since $2$ is a primitive root of $3^l$, the sequence of remainders modulo $3^l$ of its positive powers, $(2^p \mod 3^l)_{p\geq 0}$, generates $\mathcal G$, as does the sequence of its negative powers $(2^{-p} \mod 3^l)_{p\geq 0}$. Finally, since $3^l-1$ is co-prime with $3^l$, it is an element of $\mathcal G$, and $(3^l-1)\mathcal G = \mathcal G$. Therefore, the sequence $(2^{-p} \cdot (3^l-1) \mod 3^l)_{p\geq 0}$ generates $\mathcal G$, and is periodic with period $\Phi(3^l)$.
\end{proof}

%We can now state and prove the following theorem.

\begin{theorem}\label{theo:kkprime}
 Let $H_l$ be the sequence of parities of $(a_{l,p})_{0\leq p \leq \Phi(3^l)}$, written in reverse order: $$H_{l,p} = a_{l,\Phi(3^l)-p-1} \mod 2,$$ 
Let $n$ be an odd positive integer larger than $1$, with binary representation $B_n$. Let $2^l$ be the largest power of two dividing $n+1$, and $S_l$ the binary sequence $S_l=(H_l)^m1^l$, with $m$ chosen so that $|S_l|>|B_n|$, then
\begin{enumerate}[nosep]
    \item[i.] $B_n$ shares with $S_l$ a common suffix of length $s \geq l+1$, 
    \item[ii.] the long Collatz step of $n$ verifies $k=l$ and $k'=s-l$.
\end{enumerate}
\end{theorem}

\begin{proof}
\\
{\it (i)} For any $l$, the rightmost bit of $H_l$ is the parity of $a_{l,0}= 3^l-1$, which is $0$. Therefore, the rightmost $l+1$ bits of $S_l$ are $01^l$. Since $n+1$ is divisible by $2^l$ but not by $2^{l+1}$, we can write $n+1=2^l \cdot (2m+1)$, and $n=2^{l+1}\cdot m+2^l-1,$ so the binary representation of $n$ can be written $B_n=B_m+01^l$. This agrees with $S_l$ on the $l+1$ last bits. Thus, if $s$ is the length of the longest common suffix between $B_n$ and $S_l$, we have $s\geq l+1$.  
\\
\\
{\it (ii)} Notice that the sequence $a_l$ satisfies the recurrence:
\begin{align*}
    a_{l,0} &= 3^l-1,\\
    a_{l,i+1} &= \frac{a_{l,i}}2, \text{      if } a_{l,i} \text{ is even,}\\
    a_{l,i+1} &= \frac{a_{l,i}+3^l}2,  \text{     if } a_{l,i} \text{ is odd. (Note: we always have } 0 \leq a_{l,i}<\Phi(3^l)).\\
\end{align*}

Let $k$ and $k'$ be the lengths of the two loops in the computation of $\kappa(n)$. Since $2^l$ is the largest power of two dividing $n+1$, we must have $k=l$ (Section~\ref{sec:long_collatz}). Let us compute $\kappa(n)$, we have $c_0=n=2^{l+1}m+2^l-1,$ and $B_n=B_m+01^l.$ After $k=l$ up-steps, we have $$c_{l}=3^l .2m +3^l-1.$$ 
The $k'$ next Collatz steps are divisions by $2$. Let us write for all $0\leq i \leq k'$, $$c_{l+i}=3^l p_i + q_i,$$ 
with $p_i \geq 0$ and $0 \leq q_i < 3^l$.

We have $p_0=2m$, $q_0=3^l-1=a_{l,0}$, $p_1=m$ and $q_1=\frac{3^l-1}{2}=a_{l,1}$.

Since all $c_{l+i}$ for $i<k'$ must be even, and $$c_{l+i+1}=\frac{c_{l+i}}{2},$$ $p_i$ and $q_i$ must have the same parity. Let us show by induction that $$q_i=a_{l,i},$$ $$B_n = B_{p_i}+ h_{i-1}\dots h_1h_01^l,$$ with  $h_j = H_{l,\Phi(3^l)-j-1}= a_{l,j} \mod 2,$ for all $i\leq k'$. \\
\\
Since $q_0=a_{l,0}=3^l-1,$ we have $h_0=0$, and $B_n=B_{2m}+1^l=B_{p_0}+1^l.$ Also, $q_1= \frac {q_0}2 = a_{l,1},$ and $B_n=B_m+01^k=B_m+h_01^k.$ 
The property is true for $i\leq 1$.
\\
\\
Suppose that for $i<k'$, $q_i=a_{l,i}$ and $B_n=B_{p_i}+ h_{i-1}\dots h_1h_01^l$. We consider two cases: 
\begin{enumerate}
    \item $q_i$ is even: $q_{i+1}=\frac{q_i}2=\frac{a_{l,i}}2=a_{l,i+1}$. Since $q_i$ is even, $p_i$ must be even, and since $$c_{l+i+1}=3^l p_{i+1}+q_{i+1}=\frac{c_{l+i}}{2} = \frac 1 2 (3^l p_i+ q_i)=3^l \frac{p_i}{2}+q_{i+1},$$ we must have $p_{i}=2p_{i+1}$, and, since $h_{i}=a_{l,i}\mod 2 = 0$,  $$B_{n}=B_{p_i}+ h_{i-1}\dots h_1h_01^l=B_{p_{i+1}}+ 0h_{i-1}\dots h_1h_01^l=B_{p_{i+1}}+h_{i}h_{i-1}\dots h_1h_01^l.$$ 
    \item $q_i$ is odd: then $p_i$ must be odd, we write $p_{i} = 2p_{i+1}+1$, and $$c_{l+i+1}=3^l p_{i+1}+q_{i+1}=\frac{c_{l+i}}{2} = \frac 1 2 (3^l p_i+ q_i)=3^l p_{i+1}+\frac{3^l+q_i}{2}=3^lp_{i+1}+a_{l,i+1},$$
    therefore, $q_{i+1}=a_{l,i+1}$, and, since $h_{i}=a_{l,i}\mod 2 = 1$, $$B_{n}=B_{p_i}+ h_{i-1}\dots h_1h_01^l=B_{p_{i+1}}+ 1h_{i-1}\dots h_1h_01^l=B_{p_{i+1}}+h_{i}h_{i-1}\dots h_1h_01^l.$$
\end{enumerate}
This concludes the induction. We have shown that for all $i\leq k'$, $q_i=a_{l,i}$, and that the binary representation of $n$ is $B_n=B_{p_{k'}}+h_{k'-1}\dots h_1h_01^l.$

Note that $h_{k'-1}\dots h_1h_01^l$ is the suffix of length $k'+l$ of $S_l$. This proves that $s \geq k'+l$.
Finally, since $c_{l+k'}=\kappa(n)$, $c_{l+k'}$ must be odd, and the parities of $p_{k'}$ and $q_{k'}=a_{l,k'}$ must be different. Now, $p_{k'} \mod 2$ is the $k'+l+1$ rightmost bit of $B_n$, $a_{l,k'} \mod 2$ is the $l+k'+1$ rightmost bit of $S_l$. This proves that $B_n$ and $S_l$ must disagree at that point, and that we must have $s<k'+l+1$. Therefore, $s=k'+l$.
\end{proof}

Theorem~\ref{theo:kkprime} states that, for any odd integer $n$, the number of steps in the two loops of the long Collatz step calculation, $k$ and $k'$, can be read from $B_n$, the binary representation of $n$, by matching it with a family of binary sequences $S_l=(H_l)^p1^l$, for $l>0$. The $H_l$ are binary sequences of length $\Phi(3^l)$, formed from the parities of the powers of $2$ modulo $3^l$. Their first values are $H_1=10$, $H_2=111000$ and $H_3=111101101000010010$. For instance, 
an integer with binary representation ending in $000011_2$ must have $k=2$, and since $S_2=(111000)^p11$, it matches $S_2$ on the last $s=5$ bits, so $k'=5-3=3$. An integer with binary representation ending in $001101000010010111$ ($k=3$) matches the last $17$ bits of $S_3=(111101101000010010)^p111$, therefore $k'=14$. $S_l$ is the binary representation of an odd integer that has $k=l$ and $k'$ maximal. If the binary representation of $n$ matches the last $p>l$ bits of $S_l$, then $k' \geq p-l$. 

The sequences $H_l$ appear in other works on the Collatz conjecture. \citet{colussi2011} relates them with the convergence classes of the Collatz sequence: integers reaching $1$ after $l$ iterations of the transformation $n \to (3n+1)/2^h$. \citet{sterin2020} relates them with the Collatz ancestors of particular integers.

Most importantly, theorem~\ref{theo:kkprime} provides an explanation for the learning pattern. For any pair $(k,k')$, all odd integers $n$ for which $$\kappa(n)=\frac{{(\frac3 2)}^k (n+1) -1}{2^{k'}}$$
share a binary suffix, the $k+k'+1$ last bits of $S_k$. For $k=k'=1$, this common suffix is $001_2$, for $k=1$, $k'=2$, it is $1101_2$, for $k=2$, $k'=1$, it is $1011_2$. Every class in the learning pattern corresponds to a specific pair $(k,k')$: the learning pattern amounts to learning the long Collatz step one pair $(k,k')$ at a time.
\\

\begin{itemize}[nosep]
\item Models achieving $25\%$ accuracy ($B=3$) have learned to predict the long Collatz successors for $(k,k')= (1,1)$ ($25\%$ of all input).
\item Models achieving $37\%$ accuracy ($B=11$)  predict $(1,1)$ and $(2,1)$ ($37.5\%$ of input).
\item Models achieving $55\%$ accuracy ($B=5, 13, 19, 29, 35, 37, 53$)  predict $(1,1)$, $(2,1)$, $(1,2)$ and $(3,1)$, ($56.25\%$ of input).
\item Models achieving $71\%$ accuracy predict  $(1,1), (2,1), (1,2), (3,1), (1,3), (2,2)$ and $(4,1)$ ($71.825\%$ of input).
\end{itemize}

Notice that classes are learned in increasing order of $k+k'$, a consequence of Theorem~\ref{theo:kkprime}.

\begin{corollary}\label{corr:prob}
Over all odd positive integers, the random variables $k$ and $k'$ are independent and their distribution verifies: $$\mathcal P(k=a)=\frac 1 {2^a},$$ $$\mathcal P(k'=b | k=a)=\frac 1 {2^b},$$ $$\mathcal P(k=a, k'=b)=\frac 1 {2^{a+b}}.$$ 
\end{corollary}

Since $k$ and $k'$ can be read from $B_n$, it comes to no surprise that models using power-of-two bases learn best, and that even-base models perform better than odd-base models.

\section{Patterns of model failures}\label{sec:errors}

In the previous section, we considered the inputs that are correctly predicted, and noticed that they correspond to integers associated with small values of $k$ and $k'$. For these $n$, the model correctly predicts $\kappa(n)$. We now turn to model failures: inputs $n$ that are not predicted as $\kappa(n)$.

Many prior works on transformers described \emph{hallucination}: models outputting irrelevant predictions when they fail. This phenomenon is well-documented in large language models, pre-trained without supervision on large text corpora. Prior work suggests that math transformers, smaller models trained with supervision, do not hallucinate. Models trained to predict the eigen decomposition of symmetric real matrices predict, even when they fail, the correct eigenvalues, and unitary, albeit non-orthogonal, eigenvectors in most cases \citep{charton2022}. Errors in models trained to predict GCD can be explained by three deterministic rules~\citep{charton2024gcd}. 

For the long Collatz step, we will show that most model errors can be ascribed to a small number of predictable cases. In fact, we will propose a generalized version of the learning pattern that accounts for most model predictions: almost all correct predictions, and about $90\%$ of errors. In all experiments, we investigate the predictions of trained models on a test set of 100,000 examples. For each input $n$, we compare the model prediction $p$ to the target value $t=\kappa(n)$, and especially  investigate the distribution of the ratio $r=p/t$.

\subsection{Errors in the best performing models}\label{sec:close_errors}

Let us first consider the models using bases $24$ and $32$. They achieve $99.7\%$ accuracy, $243$ and $265$ errors, respectively, on a test set of 100,000. As expected, the models fail for large values of $k$ and $k'$. In base $24$, inputs with $k\geq 11$ are never correctly predicted, and only $43\%$ of inputs with $k=10$ are correctly predicted. For $k=9$, there are $14$ errors and $41$ correct predictions for $k'=1$, $26$ errors and $5$ correct predictions for $k'=2$. For $k'\geq 3$, all predictions are incorrect. In base $32$, all inputs with $k>9$ are incorrectly predicted, and account for  $212$ of $265$ errors. Inputs with $k=9$ account for $39$, almost all remaining errors. 

In almost all errors, the ratio $r=p/t$ is very close to $1$. In base $24$, $93\%$ of incorrect predictions ($225$ out of $245$) are within $3\%$ of the correct value of $\kappa(n)$, $84\%$ ($204$) are within $1\%$ and $46\%$ ($111$) within $0.1\%$.
In base $32$, $93\%$ ($247$) are within $1\%$ of $\kappa(n)$. Whereas these failures happen for the same inputs $n$, model predictions vary for different bases. 

We investigate these close misses by comparing the sequence predicted by the model (the base-$B$ representation of $p$) with the base-$B$ representation of $\kappa(n)$. In most cases, they agree on more than half of the digits: the beginning and end of the two sequences are the same. For example, in base-$24$, the model predicts  \texttt{[4, 21, 20, 6, 7, 8, 8, 2, 20, 1]} instead of \texttt{[4, 21, 20, 6, 7, 3, 20, 2, 20, 1]}: the five first and three last digits are correct. Incorrect predictions are not only close to the target values, but they share the same residuals modulo $24^p$ for small values of $p$.

These observations generalize to all bases where models achieve high accuracy. Table~\ref{tab:error_goodbases} presents, for bases divisible by $8$ or $12$, the average discrepancy between correct and predicted sequences for all model errors. On average, about half of the tokens are correct, evenly shared between the beginning and the end of the sequence. These results indicate that models errors are not random hallucinations, and suggest that the models do a better job at predicting the long Collatz step than their accuracy figures suggest.

\begin{table}[t]
    \small
    \centering
    \begin{tabular}{lccccc}
        \toprule
          & Errors & Prediction length & Correct tokens & Correct prefix & Correct suffix  \\
        \midrule
        Base 24 & 243 & 9.4 & 4.5 & 2.1 & 2.3\\
        %Apex & 243 & 10.1 & 5.6 & 2.1 & 3.2\\
        %\midrule
        Base 32 & 265 & 8.9 & 4.2 & 1.9 & 2.1 \\
        %Apex & 265 & 9.4 & 4.6 & 1.9 & 2.4 \\
        %\midrule
        Base 16 & 259 & 11.1 & 5.8 & 2.7 & 2.7 \\
        %Apex & 259 & 11.6 & 6.1 & 2.7& 3.1 \\
        %\midrule
        Base 8  & 479 & 14.5 &  8.0 & 3.5 & 3.6 \\
        %Apex & 479 & 15.1 & 8.3 & 3.5 & 4.2  \\
        %\midrule
        Base 48  & 399 & 7.7 & 4.0 & 1.9 & 2.0  \\
        %Apex & 399 & 8.3 & 4.7 & 1.8 & 2.8 \\
        %\midrule
        Base 36  & 287 & 7.7 & 3.3 & 1.5 & 1.6 \\
        %Apex & 287 & 8.7 & 4.6 & 1.6 & 2.7 \\
        %\midrule
        Base 12  & 315 & 11.7 & 6.1 & 2.6 & 2.9 \\
        %Apex & 315 & 12.7 & 7.5 & 2.6 & 4.2\\
       \bottomrule
    \end{tabular}
    \caption{\small Prediction errors for different good bases. Average length of predictions, number of tokens in agreement between prediction and target. }
    \label{tab:error_goodbases}
    \end{table}

\subsection{Power-of-two errors in odd-base models}

The long Collatz successor of an odd integer is always odd. In even-base models, this property is respected. All model prediction for bases $24$ and $32$ are odd. The largest number of even predictions, in an even-base models is $843$ for base $38$: $0.8\%$ of model predictions, and less than $5\%$ of errors. In an even base, the parity of outputs is easy to learn and control, as it only depends on the last token in the predicted sequence (which can only take values on half of the vocabulary).

No such shortcut for parity exists in odd bases\footnote{Rules for parity testing exist in odd bases, but they require a weighted modular sum over all digits, a task transformers struggle on~\citep{saxena2025}.}, and parity errors in odd-base models are frequent. In fact, they are \emph{too frequent}. In base $27$ ($71\%$ accuracy),
$25\%$ of test set predictions (25,082 out of 100,000) are even, and account for $88\%$ of model errors (all even predictions are errors). Similar proportions are found for other odd bases ($85\%$ for bases $25$ and $29$, $87\%$ for base $31$). The prevalence of even predictions in odd-base model is very surprising: a model guessing parity at random would only produce $50\%$ even predictions.

A look at the ratio $r=p/t$ settles this curious case. In about $70\%$ of model failures (Table~\ref{tab:pow2err}), the ratio is \emph{exactly equal} to a small power of $2$. Instead of predicting $t=\kappa(n)$, the model predicts, $2\kappa(n)$, $4\kappa(n)$ or $8\kappa(n)$, that is, while predicting the long Collatz step, $k'$, the number of iterations in the second loop, is underestimated by a few units.

\begin{table}[t]
    \scriptsize
    \centering
    \begin{tabular}{l|ccc||l|ccc}
        \toprule
        Base  & Accuracy &  Power-of-2 & Power-of-2 / errors & Base & Accuracy & Power-of-2 & Power-of-2 / errors\\
        \midrule
        33&88.6&8.0&69.8 & 7&70.2&20.4&68.5\\
        9&82.6&13.8&79.0  & 25&70.2&20.7&69.4\\
        45&82.6&13.6&77.9 & 57&65.3&27.7&79.9\\
        15&82.0&13.4&74.7 &  43&59.4&28.5&70.0\\
        17&81.9&13.6&75.4 & 13&56.9&28.6&66.5\\
        49&71.7&21.6&76.5 & 37&56.0&30.7&69.7\\
        47&71.5&21.3&74.7 & 19&55.9&30.9&70.2\\
        21&71.4&21.3&74.4 & 55&55.7&30.4&68.7\\
        27&71.4&21.6&75.5 & 29&55.6&31.0&69.8\\
        51&71.3&21.2&74.1 & 53&55.6&30.8&69.3\\
        39&71.1&21.1&72.7 & 35&55.5&31.1&69.8\\
        41&71.0&21.3&73.7 & 5&55.4&30.3&68.0\\
        31&71.0&21.2&73.0 & 11&37.3&37.2&59.3\\
        23&70.9&21.0&72.1 & 3&24.8&24.9&33.1\\
        \bottomrule
    \end{tabular}
    \caption{\small \textbf{Power-of-two errors} Accuracy (\%), Power-of-two errors (\% of all predictions), Power-of-two errors (\% of all errors).}
    \label{tab:pow2err}
\end{table}

These \emph{power-of-two errors} happen for specific values of $n$, associated with $k$ that the models can predict (i.e. inputs with $k'=1$ are correctly predicted for this $k$), but large $k'$, that are predicted as the largest value the model can predict. For instance, in base $27$, the model correctly  predicts inputs with $k=1$ and $k'\leq 3$, $k=2$ and $k'\leq 2$, and $k=3$ or $4$ and $k'=1$. Almost all ($99.9\%$) power-of-two errors happen for inputs with $k=1, k' \geq 4$, $k=2, k'\geq 3$, and $k=3$ or $4$ and $k'\geq 2$ (and $99\%$ of model predictions for these inputs are power-of-two errors). In all power-of-two errors, the model predicts $k'=3$ for $k=1$, $k'=2$ for $k=2$ and $k'=1$ for $k=3$ or $4$. 

In other words, power-of-two errors happen for values of $k$ that the model can predict correctly up to a certain value of $k'=l_k'$. Inputs associated with this value of $k$ are predicted as $p=2^{k'-l_k'}\kappa(n)$ (the largest $k'$ that the model has learned so far).

These observations apply to all odd bases. In base $9$ ($83\%$ accuracy), models correctly predict inputs with $k$ up to $5$. Power-of-two errors happen for $k=1$ and $k'\geq 5$ (the model then predicts $p=2^{k'-4}\kappa(n)$), $k=2$ and $k'\geq 4$ (predicting $p=2^{k'-3}\kappa(n)$), $k=3$ and $k'\geq 3$ ($p=2^{k'-2}\kappa(n)$), and $k=4$ or $5$ and $k'\geq 2$ ($p=2^{k'-1}\kappa(n)$). In base $11$ ($37\%$ accuracy), inputs with $k=1$, $k'=1$ and $k=2$, $k'=1$ are correctly predicted, all other inputs with $k\leq 2$ are power-of-two-errors, predicted as $p=2^{k'-1}\kappa(n)$. Inputs with $k>2$ are almost never correct predictions or power-of-two errors.

Power-of-two errors generalize the learning pattern, for odd-base models. For every model, there exists a value $k_{\text{max}}$, the largest $k$ for which inputs associated with the pair $(k_\text{max},1)$ are predicted with high accuracy. For $1\leq k \leq k_\text{max}$, let $l_k'$ be the largest value of $k'$ that the model can predict, then almost all inputs with $k' \leq l_k'$ are correctly predicted, and almost all inputs with $k'>l_k'$ are power-of-two errors, predicted as $p=2^{k'-l_k'} \kappa(n)$. 

Almost all inputs associated with $k> k_\text{max}$ are neither correct prediction, nor power-of-two errors. We call these \emph{hard} errors. On these, model predictions are even in about $50\%$ of cases. Therefore, power-of-two errors account for the larger number of even predictions in odd-models.

Once again, these results indicate that the model does a much better job at learning $\kappa(n)$ than its accuracy suggests. For almost all inputs with $k \leq k_\text{max}$, the model predicts $\kappa(n)$ up to a predictable power-of-two factor. For base $27$, these account for $93\%$ of model predictions. For base $19$, they account for $87\%$.

\subsection{Near power-of-two errors in even bases}

Figures~\ref{fig:ratio26} and~\ref{fig:ratio38} present the distribution of ratios ($r=p/t$) for all errors in base $26$ and $38$. Like accuracy levels, ratios concentrate around a small number of discrete values. Most of the ratios are very close to powers of two, $r\approx 2^l$, with $l\geq 1$. They are distributed as a power law:  $\mathcal P(r\approx 2^l) = \frac{K}{2^l}$. This is observed in all even bases, except those divisible by large powers of two, discussed in Section~\ref{sec:close_errors}. 

\begin{figure}[b]
    \centering
    \scriptsize
    \begin{minipage}{0.49\textwidth}
        \centering
        \includegraphics[width=\linewidth]{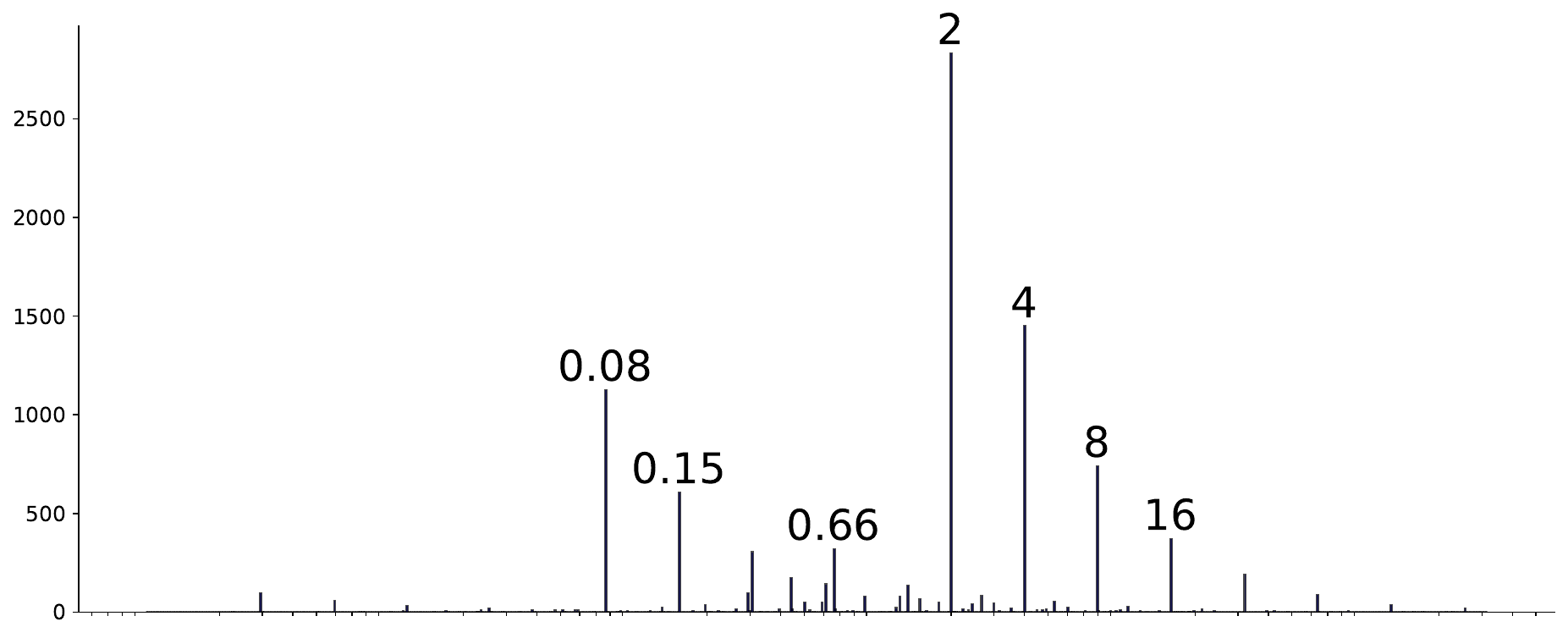} % Remplacez par le chemin de votre image
        \caption{\small Distribution of p/t: base 26 (all model errors)}\label{fig:ratio26}
    \end{minipage}\hfill 
    \begin{minipage}{0.49\textwidth}
        \centering
        \includegraphics[width=\linewidth]{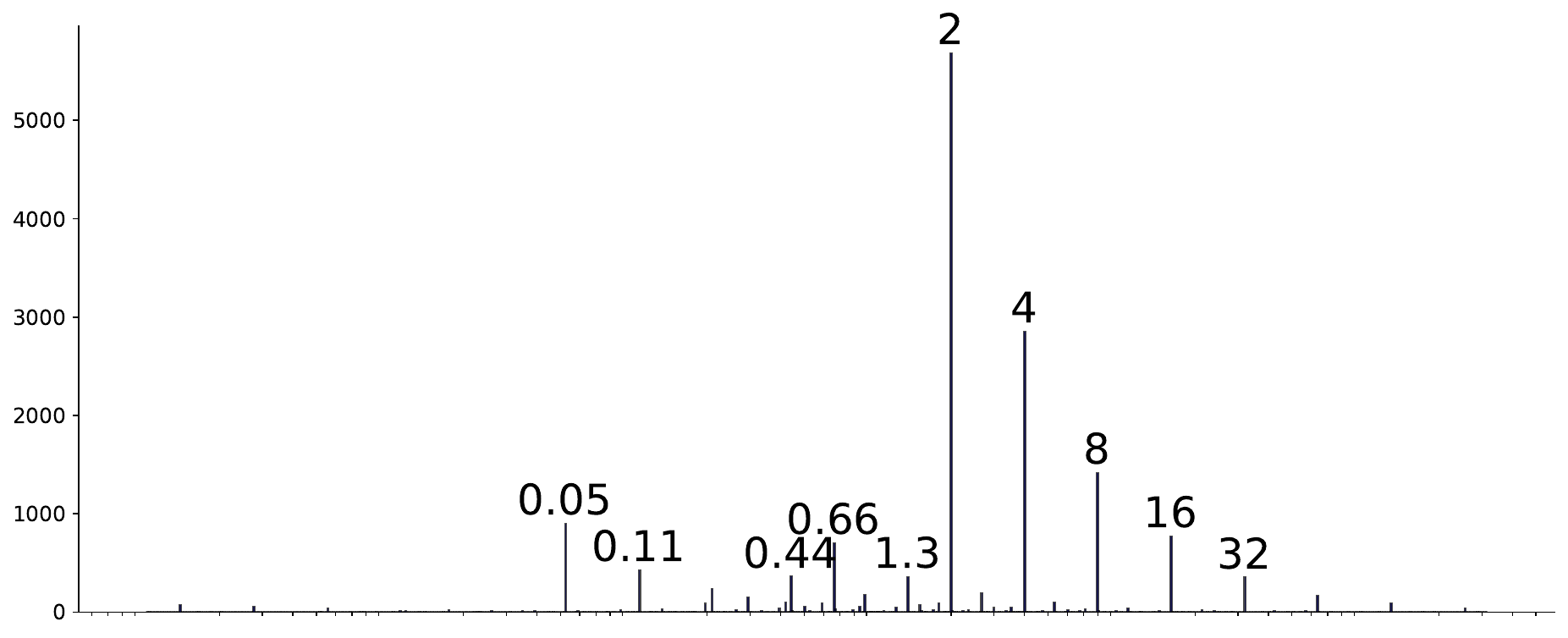} % Remplacez par le chemin de votre image
        \caption{\small Distribution of p/t: base 38 (all model errors)}\label{fig:ratio38}
    \end{minipage}
\end{figure}

Only a small fraction of these errors are power-of-two errors, that is, verify $p=2^lt$. Such exact power-of-two errors account for $328$ errors out of $10675$ in base $26$, and $561$ out of $16808$ in base $38$. All others are \emph{near power-of-two errors} (NP2E), which verify $p=2^l \kappa(n)+\epsilon$, for $\epsilon$ a small odd integer. We consider that all predictions verifying $\epsilon / p < 0.001$ are NP2E. NP2E account for $56\%$ of model errors in base $26$ ($6004$ out of $10675$), $64\%$ in base $22$, and $67\%$ in base $38$. They are the first cause of errors in non-power-of-two even-base models.

All $\kappa(n)$ in the training set are odd. In an even base, it means that all output sequences end with an odd digit. This inductive bias is learned by the model. Therefore, in almost all NP2E, the model predicts (wrongly) an odd integer, and $\epsilon$, which appears as a ``rounding error'' forced upon the model by the inductive bias, is odd. The distribution of $\epsilon$ is not random. In base $26$, $|\epsilon|$ takes the values $1, 13$ and $13^2=169$ in $80\%$ of NP2E. In base $38$, it takes the values  $1, 19$ and $19^2$, and in base $22$, the values $1, 11$ and $11^2$. Whereas the error pattern is a property of the long Collatz step, model predictions depend on the base.

In another common error pattern, the ratio $r=p/t$ has the form $\frac{2^l}{B}$. In these NP2E errors, the last odd digit in the output sequence is replaced by the \texttt{<end-of-sentence>} (EOS) token, and the output is cut short. In base $26$, we have $r \approx \frac{2^l}{B}$ in $2322$ cases ($23\%$ of errors), and $r \approx \frac{2^l}{B^2}$ in $215$ cases ($2\%$ of errors). In base $22$, these account for $16$ and $2\%$ of errors, in base $38$ for $10$ and $1\%$. 

Overall, near power-of-two errors account for $81\%$ of model errors for base $26$, $82\%$ for base $22$ and $78\%$ for base $38$. The proportion of NP2E in even bases is similar to the proportion of power-of-two errors in odd bases. 

All NP2E follow the same learning pattern as power-of-two errors in odd bases. For an input $n$, with loop lengths $k$ and $k'$, let $k_\text{max}$ be the largest $k$ that the model can predict correctly, and let, for any $k\leq k_\text{max}$, $l_k'$ be the largest value of $k'$ that the model predicts for that value of $k$. All NP2E happen when $k\leq k_\text{max}$ and $k'>l_k'$, and the power-of-two associated with the NP2E is $2^{k'-l_k'}$. In base $26$, NP2E represent $97.5\%$ of model errors for inputs with $k\leq k_\text{max}$ and $k'> l'_k$.

\subsection{Hard errors}

Power-of-two errors, exact or near, account for about $80\%$ of all model errors, and almost all errors for inputs with $k\leq k_\text{max}$, for all bases except the best performing. We now investigate the remaining \emph{hard errors}. 
Figures~\ref{fig:ratio226} to~\ref{fig:ratio19} present the distribution of $p/t$ for hard errors, for bases $26, 38, 27$ and $19$. The distributions are discrete, and take values of the form $r \approx (\frac{2}{3})^a 2^l$ with probability $$\mathcal P ( r\approx (\frac{2}{3})^a 2^l ) = \frac{K}{2^l},$$ for small positive values of $a$ and $l\geq0$. For base $27$, we also find ratios of the form $r \approx (\frac{2}{3})^a \frac{2^l}B$.

%we notice a second set of power laws, scaled down by $\frac 1 B$:  $$\mathcal P(r\approx \frac 1 B(\frac 2 3)^a 2^l) = \frac{K}{2^l}.$$

\begin{figure}[h]
    \centering
    \scriptsize
    \begin{minipage}{0.49\textwidth}
        \centering
        \includegraphics[width=\linewidth]{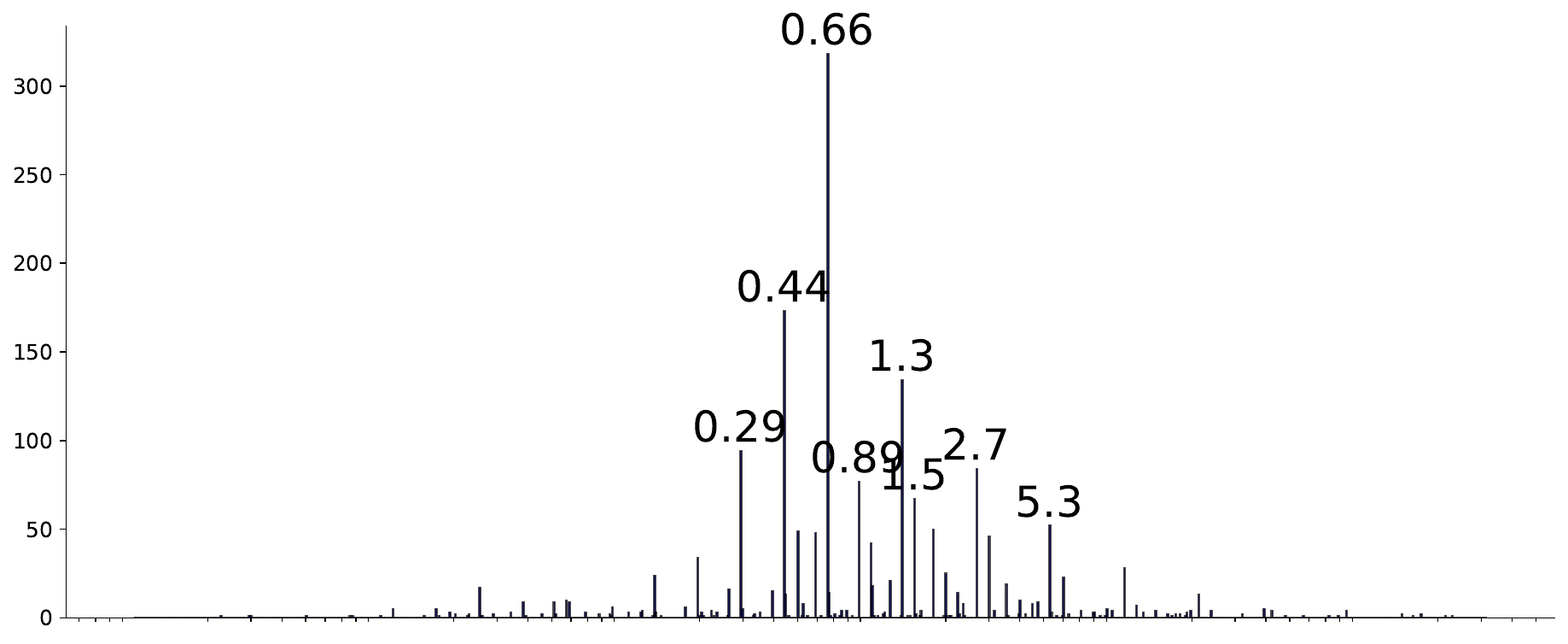} % Remplacez par le chemin de votre image
        \caption{\small Distribution of p/t: base 26 (all model errors, excluding near power of two)}\label{fig:ratio226}
    \end{minipage}\hfill 
    \begin{minipage}{0.49\textwidth}
        \centering
        \includegraphics[width=\linewidth]{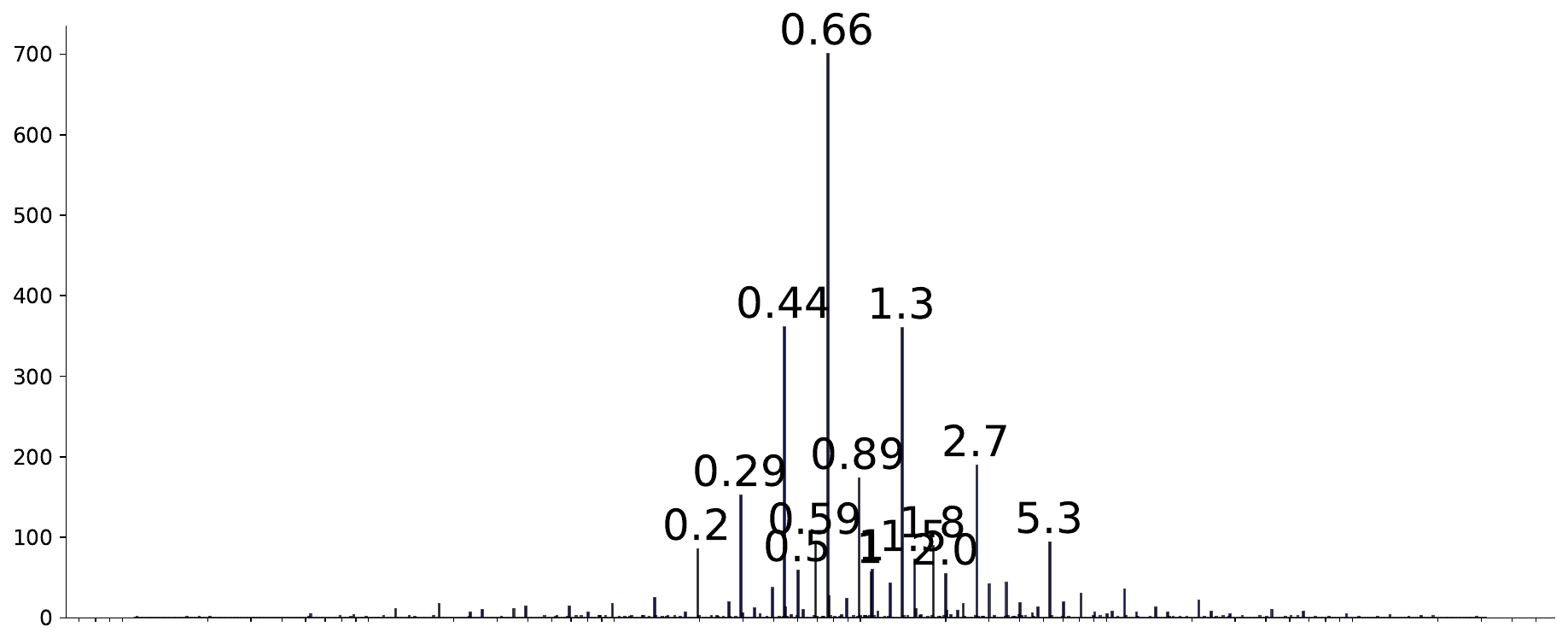} % Remplacez par le chemin de votre image
        \caption{\small Distribution of p/t: base 38 (all model errors, excluding near power of two)}\label{fig:ratio238}
    \end{minipage}
\end{figure}

\begin{figure}[h]
    \centering
    \begin{minipage}{0.49\textwidth}
        \centering
        \includegraphics[width=\linewidth]{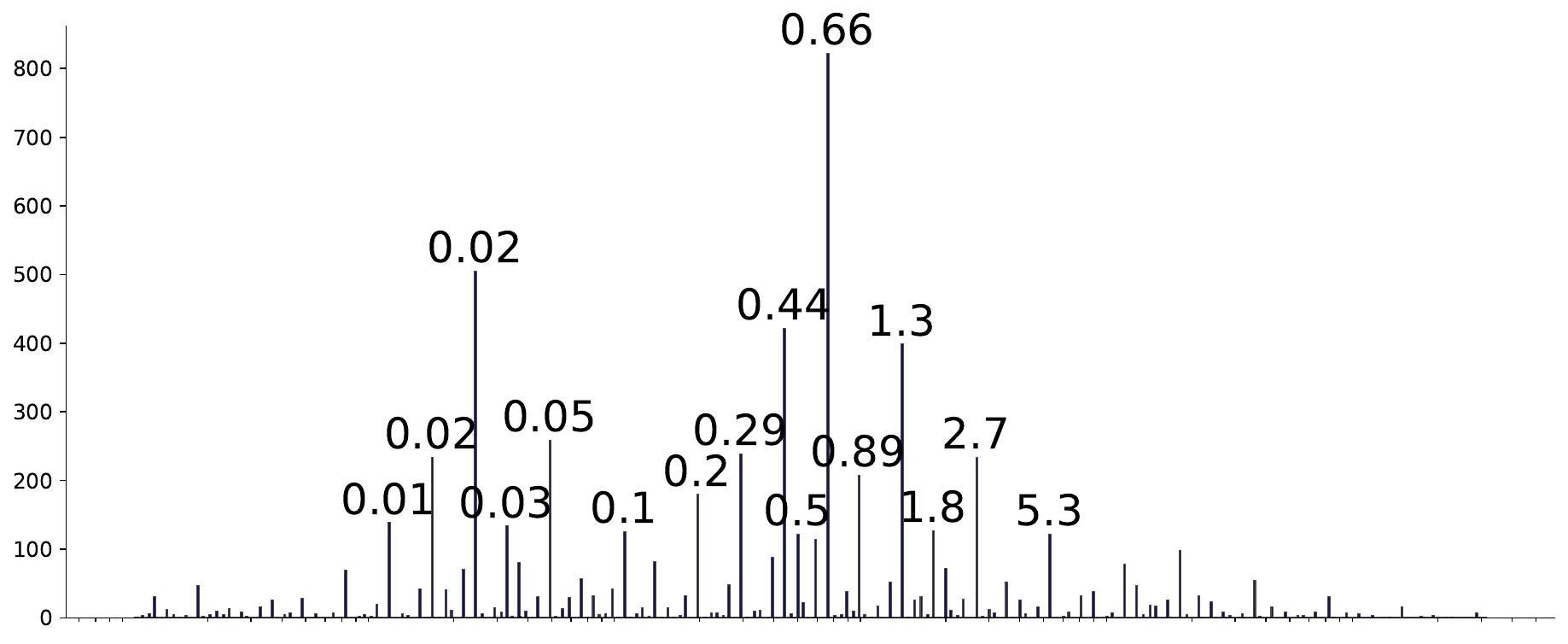} 
        \caption{\small Distribution of p/t: base 27 (all model errors, excluding power of two)}\label{fig:ratio27}
    \end{minipage}\hfill
    \begin{minipage}{0.49\textwidth}
        \centering
        \includegraphics[width=\linewidth]{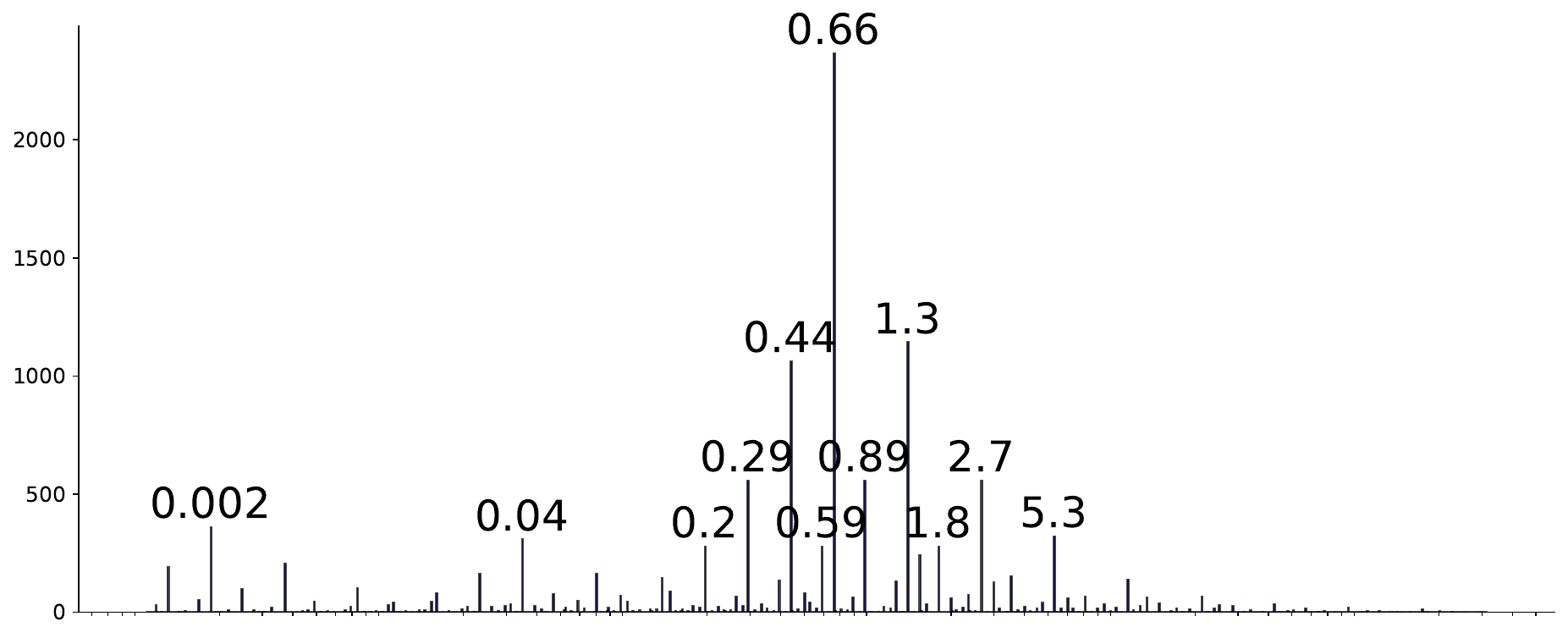} 
        \caption{\small Distribution of p/t: base 19 (all model errors, excluding power of two)}\label{fig:ratio19}
    \end{minipage}
\end{figure}

Just as the power-of-two errors happen when the model underestimates $k'$, these hard errors correspond to situations when $k$ is underestimated as $k-a$. 

In base $26$, we have $k_\text{max}=6$, and $k>k_\text{max}$ in $1545$ test examples, all predicted incorrectly. There are $36$ close misses, $36$ near power-of-two errors, and $1473$ hard errors. We have $p\approx \frac{2}{3} 2^l t$ ($k$ underestimated by $1$) in $607$ cases, and $p\approx \frac{2}{3B} 2^l t$ ($k$ underestimated by $1$, and the sequence truncated by one token) in $31$. These account for $43\%$ of hard errors. $k$ is underestimated by $2$ in $319$ cases ($22\%$ of hard errors), by $3$ in $185$, by $4$ in $65$, \&c. 

In all cases, the model predicts (up to factor $B$) as if $k=k_\text{max}$ and $k'=1$, and $$p\approx \frac{(\frac 3 2)^{6}(n+1)-1}{2},$$
this accounts for $82\%$ or hard errors, but most of the remaining $289$ errors follow similar patterns. In $66$ cases, the model predicts as if $k=7$ and $k'=1$, in $13$ $k=8$ and $k'=1$, in $8$, $k=6$ and $k'=2$. All model predictions correspond to incorrect guesses of $k$ and $k'$.

This error pattern is observed in all bases. For base $38$ ($k_\text{max}=5$), the model predicts $k=5$, $k'=1$ (up to a division by $B$) in $86\%$ of hard errors.  In base $19$ ($k_\text{max}=3$), the same pattern accounts for $90\%$ of hard errors. In base $27$ ($k_\text{max}=3$), it accounts for $92\%$.

Summarizing, for inputs with $k>k_\text{max}$, $80$ to $95\%$ of hard errors (depending on the base) are due to the model predicting $k=k_\text{max}$ and $k'=1$. Infrequently, the model prediction is cut by one (rarely two) tokens.

\subsection{A hierarchy of errors --  the learning pattern revisited}\label{sec:hierarchy}

This analysis of failure cases provides a broader perspective on the learning pattern. It appears that instead of learning $\kappa(n)$ for all $n$, or an approximation of it, all models learn a collection of functions $\kappa_{k,k'}$, that are equal to $\kappa$ over integers $n$ associated with $k$ and $k'$.

In other words, instead of learning the function $$\kappa(n)= \frac {(\frac 3 2)^{k(n)}(n+1)-1}{2^{k'(n)}},$$
where $k(n)$ and $k'(n)$ can be derived from the binary representation of $n$, as in Section~\ref{sec:theory}, 
the model learns a sequence of functions $$\kappa_{l,l'}(n)= \frac {(\frac 3 2)^l(n+1)-1}{2^{l'}},$$ derived from $\kappa$ by replacing the functions $k(n)$ and $k'(n)$ by the constants $l$ and $l'$. We have $\kappa(n) = \kappa_{l,l'}(n)$ for all $n$ in $$\mathcal D_{l,l'} = \{n=2p+1, p \in \mathbb N \mid k(n)=l, k'(n)=l' \}.$$ 

When training begins, the model predicts for all $n$ an approximation of $$\kappa_{1,1}(n) = \frac {\frac 3 2(n+1)-1}{2}.$$ This is equal to $\kappa(n)$ on $\mathcal D_{1,1}$, resulting in an accuracy of $25\%$. 

As training progresses, new values of $k$ and $k'$ are learned, in increasing order of $k+k'$. Once the model has learned all $k$ up to $k_\text{max}$, and all $k'$ up to $l_k'$ for a given value of $k$, it predicts (exactly) $\kappa_{k,k'}(n)=\kappa(n)$ for all $n$ such that $k \leq k_\text{max}$ and $k' \leq l'(k)$ (all $n$ in $\mathcal D_{k,k'}$). 
For larger $k<k_\text{max}$ and $k'>l_k'$, the model predicts $\kappa_{k,l_k'}(n)$. For $k>k_\text{max}$, it predicts $\kappa_{k_\text{max},1}(n)$. These calculations are approximate, and subject to rounding errors. In a small number of cases, the model predicts $\frac 1 B\kappa_{k,k'}(n)$, or $\frac 1 {B^2} \kappa_{k,k'}(n)$.

These rules account for more than $99\%$ of model predictions for $k \leq k_\text{max}$, and $85$ to $95\%$ for $k> k_\text{max}$. Therefore, the updated learning pattern provide an almost complete explanation of model predictions.

\section{Ablation studies}\label{sec:ablations}

We now investigate a number of questions raised by the experiments from Section~\ref{sec:base}. We first consider simpler versions of the problem: learning the first loop of the algorithm (predicting the apex instead of the long Collatz successor), and learning the loop lengths (predicting $k(n)$ and $k'(n)$). We then investigate the hypothesis that models learn the long Collatz step by first learning to translate their input into binary representation. Finally, we study the impact of different training distributions.

\subsection{Simpler versions of the problem: predicting the apex and the loop lengths}\label{sec:predkkp}

The long Collatz step is a complex arithmetic operation. It features two loops of variable lengths, $k(n)$ and $k'(n)$, that the model must learn to compute the long step. We now consider two simpler versions of this task: predicting the apex $c_{k}={(\frac3 2)}^k (n+1) -1$, which amounts to computing the first loop only, and predicting the lengths of the loops, the pair $(k,k')$. We train models with the same architecture and number of parameters as in the main experiments, on random odd integers from $1$ to $10^{12}$, encoded in base $2$ to $57$.

\begin{figure*}[!b]
\begin{minipage}{0.6\textwidth}
    \scriptsize
    \centering
    \begin{tabular}{ll}
        \toprule
         Accuracy & Bases  \\
        \midrule
        99.8 -- 99.9\% & 36, 32, 24, 48, 12, 8 \\
        99.6 -- 99.7\% & 6, 40, 4, 56, 52, 20  \\
        99.0 -- 99.5\% & 16, 18, 28, 44, 42 \\
        98.4 -- 98.7 \% & 10, 30, 22, 34 \\
        96.6 --  97\% & 54, 38, 26, 50, 14,  2 \\
        93.9\% &  46, 27, 7  \\
        87.5 -- 87.8\% & 57, 35, 23, 33, 25, 3, 9, 11, 39, 17, 45, 37, 47, 29, 49, 43 \\
        85.6 \% & 51 \\
        75.4 -- 75.5 \% & 21, 19, 15, 53, 41, 55, 31, 5 \\
        50.5\% & 13 \\
       \bottomrule
    \end{tabular}
    \small
    \captionof{table}{\small Apex prediction, model accuracy. Bases listed by decreasing accuracy, after 400 million examples.}
    \label{tab:apex}
\end{minipage}
\hfill
\begin{minipage}{0.35\textwidth}
\includegraphics[width=0.92\textwidth]{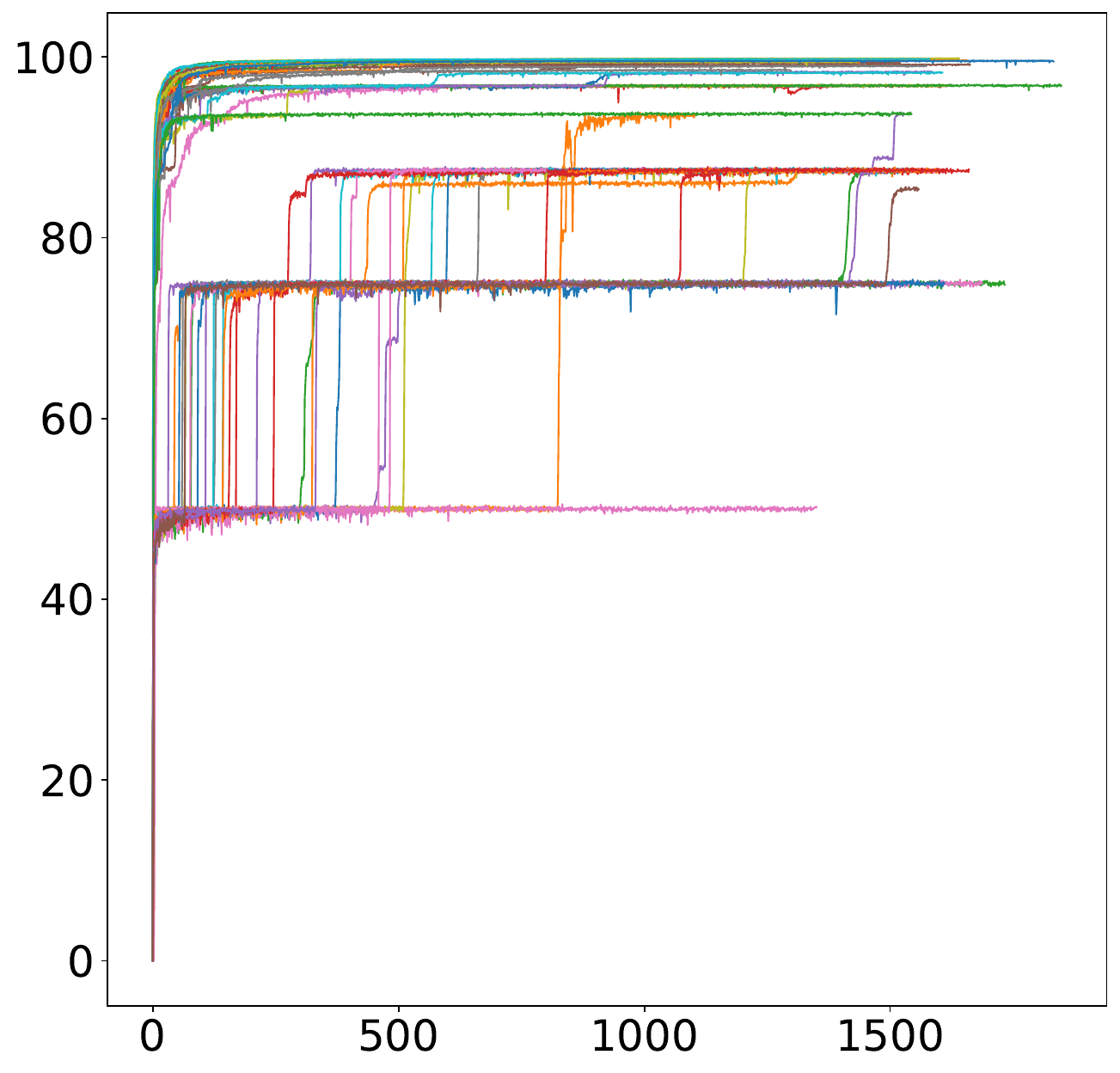}
\captionof{figure}{\small  Leaning curves for apex prediction. Bases 2 to 57.}
    \label{fig:apex}
\end{minipage}\\
\end{figure*}

\paragraph{Learning the first loop.} On this simple task, model accuracies are higher on average than when learning $\kappa(n)$, the long Collatz successor (Table~\ref{tab:apex}). Even-base models perform better than odd-base models, and the best performances are achieved by bases multiples of $12$. Learning curves (Figure~\ref{fig:apex}) keep their step-like shape, and accuracy levels take a few quantized values, close to  $1-\frac 1 {2^n}$ ($50, 75, 87.5, 93.8, 96.8, 98.4,  99.2, 99.6, 99.8\%$). This suggests a learning pattern where models learn to predict all input up to a certain $k$. 

An analysis of model errors confirms this hypothesis. Base-$13$ models ($50\%$ accuracy) predict $p = \frac{3n+1}2$ for almost all $n$, as if $k=1$ for all $n$. Models with $75\%$ accuracy (e.g. base $21$) predict $p = \frac{3n+1}2$ for $k=1$ and $p=\frac{9n+5}4$ for $k\geq 2$. Models with $87.5\%$ accuracy predict all $k\leq 3$ and so on. The best models predict all inputs with $k$ up to $9$. 

Interestingly, there is no qualitative difference between models trained to predict the long Collatz step and models trained to predict the apex. The presence of only one loop, the length of which is easier to read from the binary representation of $n$, does not result in a different learning pattern.

\paragraph{Learning loop lengths.} Predicting $k$ and $k'$ eliminates the arithmetic calculations in the task. The model no longer needs to learn the transformations  $n \to \frac {3n+1}2$ and $n \to \frac n 2$, and the task becomes a pattern recognition problem: identifying the class $(k,k')$ associated with the input $n$ (from a small set of possible classes). We know from Section~\ref{sec:theory} that this can be deduced from the $k+k'+1$ rightmost bits in the binary representation of $n$. 

\begin{wrapfigure}{r}{0.45\textwidth}
%\vspace{-0.25cm}
\includegraphics[width=0.9\linewidth]{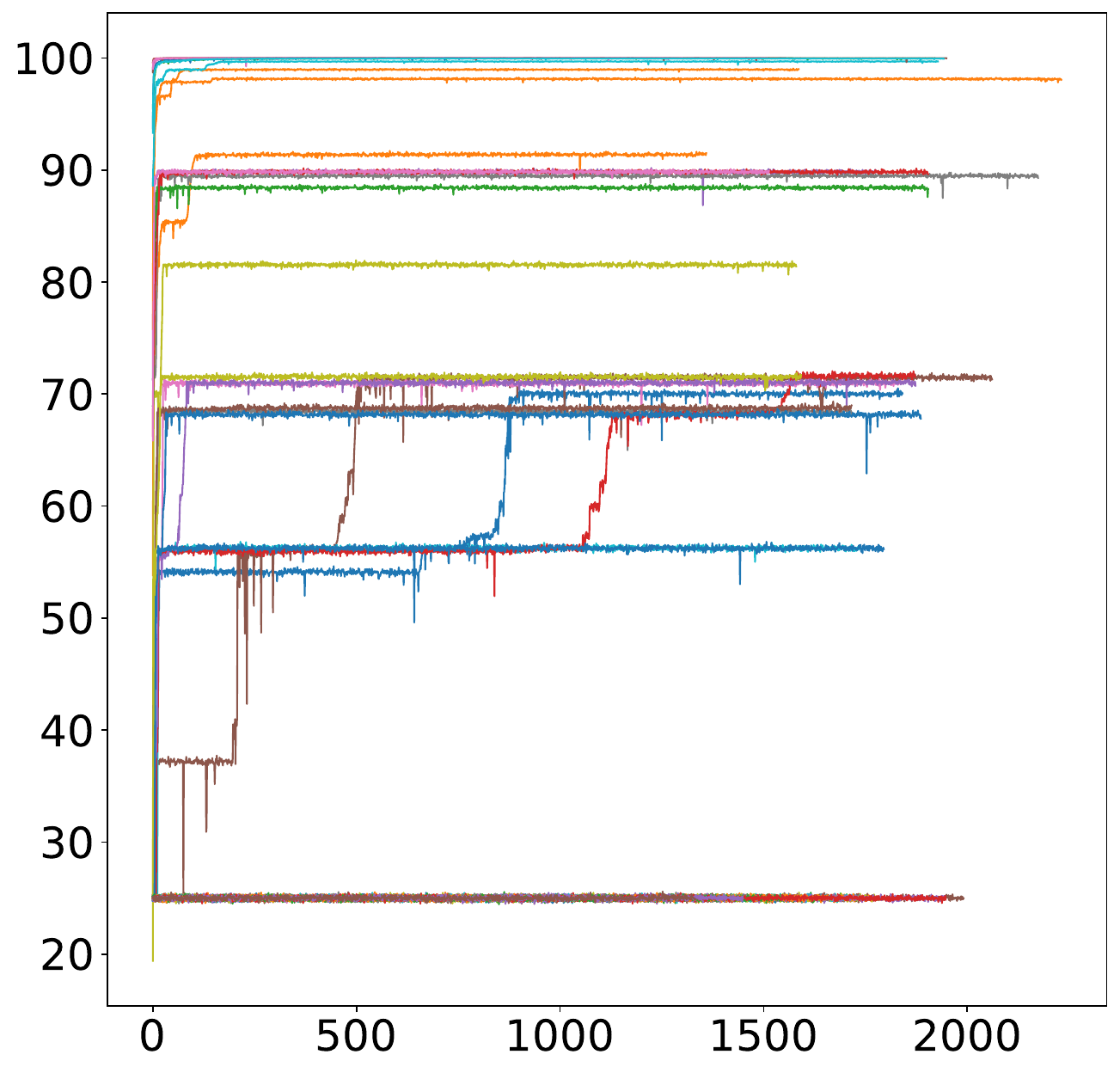}
\caption{\small Learning curves for k and k' prediction. Bases 2 to 57.}
    \label{fig:kkprime}
%\vspace{-0.25cm}
\end{wrapfigure}

The learning curves are step-shaped and accuracy take quantized levels that roughly correspond to those observed when predicting $\kappa(n)$ (Figure~\ref{fig:kkprime}). Models using base $2, 4, 8, 16$ or $32$ achieve $100\%$ accuracy. In these bases, $k$ and $k'$ can be read from the (binary) input sequence. Base-$48$ and $24$ models achieve almost perfect accuracies ($99.99$ and $99.98$), but all other bases perform worse than when predicting the long Collatz successor. In particular, odd-base models fail to learn the task and achieve $25\%$ accuracy by always predicting the most frequent outcome, $k$=$k'$=$1$. Even-base models not a power of $2$ or divisible by $24$ reach higher accuracies, but always perform worse than when predicting $\kappa(n)$. For instance, base-$30$ and base-$46$ models achieve $56\%$ accuracy, vs $89\%$ for $\kappa(n)$.

The learning pattern applies: $k$ and $k'$ are learned in increasing order of $k+k'$, yielding the familiar quantized accuracy levels at $25$, $37$, $56$, $71$, $81$, $88$, and $91\%$ (Table~\ref{tab:kkprime}). However, a fine structure appears, as accuracy levels split into bands of close, quantized values, where the largest correct values of $k'$ are predicted with less than $99\%$ accuracy.

For instance, models using bases $14$, $50$ and $54$ achieve $72\%$ accuracy by predicting all inputs with $k=1$, $k'\leq 4$, $k=2$, $k'\leq 2$ and $k=3, 4$, $k'=1$, with over $98\%$ accuracy. Base-$34$ and $22$ models predict the same pairs, but $(1,3)$, $(2,2)$ and $(4,1)$ are predicted with $95\%$ accuracy only, for a model performance of $71.5\%$. In the base-$38$ model these three pairs are predicted with $90\%$ accuracy, for a performance of 
$70\%$, and in base $46$ and $30$ with $75\%$, for a performance of $69\%$. These suggest a second-order mechanism for the learning of new $(k,k')$ classes. Here, all three classes are learned simultaneously, but in discrete steps.

\begin{table}[t]
    \scriptsize
    \centering
    \begin{tabular}{lll}
        \toprule
         Bases & Accuracy & Correct predictions (k-k')  \\
        \midrule
        40  & 99.8\% & all  (k-k') for $k+k'\leq 12$, (12,1)\\
        12  & 99.1\% & all  (k-k') for $k+k'\leq 10$, (10,1)\\
        56 & 98.3\% & all  (k-k') for $k+k'\leq 9$, (9,1)\\
        20, 44, 28, 52 & 90.1\% &  (1-1,2,3,4,5), (2-1,2,3,4), (3-1,2,3), (4-1,2), (5,1), (6,1)\\
       36 & 88.8\% & (1-1,2,3,\textcolor{red}{4,5}), (2-1,2,\textcolor{red}{3,4}), (3-1,\textcolor{red}{2,3}), (4-1,\textcolor{red}{2}), \textcolor{red}{ (5-1), (6-1) (92\%)}\\
       10 & 81.5\% & (1,1,2,3,\textcolor{red}{4}), (2,1,2,\textcolor{red}{3}), (3,1,\textcolor{red}{2}), (4,1),  \textcolor{red}{(5,1) (90\%)}\\ 
       14, 50, 54 & 72\% & (1-1,2,3), (2-1,2), (3-1), (4-1) \\
        34, 22 & 71.5\% & (1-1,2,\textcolor{red}{3}), (2-1,\textcolor{red}{2}), (3-1), \textcolor{red}{(4-1) (95\%)}\\  
        38  & 70\% & (1-1,2,\textcolor{red}{3}), (2-1,\textcolor{red}{2}), (3-1), \textcolor{red}{(4-1) (90\%)} \\
       42, 26, 18 & 69\% & (1-1,2,\textcolor{red}{3}), (2-1,\textcolor{red}{2}), (3-1), \textcolor{red}{(4-1) (75\%)} \\
        46, 30 &  56.5\% & (1-1,2), (2-1), (3-1) \\
        
       \bottomrule
    \end{tabular}
    \small
    \caption{\small Predicting k and k'. Model predictions for even bases, after 500 million examples. Accuracies below 98\% are in red.}
    \label{tab:kkprime}
\end{table}

The lower performance of models learning to predict $k$ and $k'$, compared to models predicting $\kappa(n)$, is very counter-intuitive. One would expect that a simpler task  -- predicting loop lengths -- is easier to learn than a complex one -- predicting loop lengths AND calculating $\kappa(n)$. We believe that the behavior of odd-base models, that easily learn the obvious solution $k=k'=1$, for $25\%$ accuracy, but never get out of the local minimum of the loss function, offer insight about what is happening.
In this task, the model is tasked to predict two small integers, when predicting $k=k'=1$, from uniformly distributed inputs, the model is correct on both tokens, i.e. the cross-entropy loss is zero, in $25\%$ of cases, but the model predicts at least one integer right in $75\%$ of training examples. This suggests a very flat loss landscape around the local minimum, which becomes very hard to escape. The same observation can be made for larger accuracy levels: because the distribution of $k$ and $k'$ are independent and exponentially decreasing, the loss function is very flat around local minima. In contrast, when predicting $\kappa(n)$, there is no obvious solution, that would result in flat local minima. This empirical observation deserves further story: it suggests that simple but hard-to-learn arithmetic functions may be easier to learn if they are ``relaxed'' into harder (regression) problems. Here, even if the goal is to learn $k$ and $k'$, on would prefer to train a model to predict $\kappa(n)$, from which $k$ ad $k'$ can be deduced.

\subsection{Can transformers learn base conversion?}\label{sec:base_change}

The long Collatz step is easier to learn when inputs and outputs are encoded in a base divisible by a large power of $2$. In such bases, the rightmost bits in the binary representation of $n$, which encode the loop lengths $k$ and $k'$ (Section~\ref{sec:theory}), can be deduced from the last tokens in the input sequence (specifically, if $B$ is a multiple of $2^d$, the $p$ last bits in $B_n$ are encoded in the $\lceil \frac p d \rceil$ last digits of the sequence representing $n$).  

This suggests a possible explanation for the accuracy gaps between different bases. Models might first learn to convert their input into binary representation, then learn the long Collatz step. The worse performance of odd-models would then be a consequence of the additional work of learning to convert model input into their binary representation. For output bases, we selected $2, 8, 32$ and $24$, the bases that achieved the best performance in our experiments from Section~\ref{sec:base}. For input base, we chose $4, 12, 22, 36, 42, 56, 3, 9, 27, 11, 15, 13, 31$ and $43$. For each of these $56$ input/output pairs, two models were trained over a random training set of $700$ million integers from $1$ to $10^{12}$.

The results were very disappointing: only $8$ models out of $108$ achieved an accuracy of $98\%$ or more: the $6$ models converting base $4$ into base $2, 8$ and $32$, and the $2$ models converting base $12$ into $24$. The two models converting base $4$ into $24$ achieved $47\%$ accuracy, and models converting base $36$ into $24$ achieved $11$ and $1\%$. All other models achieved less than $0.2\%$ accuracy. Models only learn to convert in the trivial case where one base is a multiple of the other.

These results refute our hypothesis, by demonstrating that base conversion is a hard task for transformers. They also confirm the surprising nature of our results on odd-base models. The learning pattern clearly indicates that odd-base models learn to classify inputs according to their binary ending (i.e. their residuals modulo $2^p$), but our experiments demonstrate that learning the binary representation of a number from its digits in an odd-base sequence is a hard task.

\subsection{Training from different distributions}\label{sec:uniform}

Corollary~\ref{corr:prob} from Section~\ref{sec:theory} offer a possible explanation for the learning pattern. Since the probability of a uniformly distributed input $n$ being associated with $k$ and $k'$ is $2^{-(k+k')}$, inputs $n$ associated with $k=k'=1$, the first class in the learning pattern, account for $25\%$ the training set. The second class ($k=2$, $k'=1$) account for $12.5\%$, inputs with $k=k'=3$ for $1.5625\%$, and inputs with $k\geq 5$ for $6.25\%$. Thus, the learning pattern might just be a consequence of the distribution of training data, and large values of $k$ and $k'$ would be hard to learn because they are less common. Such a situation was observed in previous work~\citep{charton2024gcd}, and large increases in performance could be achieved by training the model on a more balanced distribution. 

\begin{table}[t]
    \scriptsize
    \centering
    \begin{tabular}{cccc|cccc}
        \toprule
         Even bases & Uniform  & Log-uniform  & Baseline & Odd bases & Uniform & Log-uniform & Baseline  \\
        \midrule
         32 & 93.9 & 97.6 & 99.8 & 33 & 0 & 0.2 & 89.1 \\
         16  & 98.5  & 99.2 & 99.7 & 9 & 0.1 & 0.3 & 82.8 \\
         12 & 99.8 & 99.6 & 99.7 & 17 & 0 & 0 & 82.3 \\
         36 & 98.9 & 99.5 & 99.7 & 27 & 1.1 & 1.3 & 71.8 \\
         24 & 100& 100 & 99.7 & 45 & 0 & 0 & 71.8 \\
         48 & 100 & 99.9 & 99.6 & 47 & 0 &0  & 71.7 \\
         4 & 96.7 & 97.0 & 99.4 & 21 & 0 & 0 & 71.6 \\
         56 & 98.8 & 99.5 & 99.3 & 51 & 0 &0  & 71.6 \\
         8 & 98.7 & 97.0 & 99.3 & 41 & 0 & 0 & 71.5 \\
         40 & 98.6 & 99.4 & 99.2 & 49 & 0 & 0 & 71.3\\
         20 & 98.6 &97.0 & 98.9 & 15 & 0 & 0 & 71.2 \\
         6 & 98.7 & 99.7 & 98.8 & 25 & 0 & 0 & 70.8\\
         28 & 97.6 & 96.5 & 98.7 & 31 & 0 & 0 & 67.9 \\
         2 & 97.9 & 98.4 & 97.7 & 23& 0 & 0 & 62.0 \\
         18 & 98.2 & 99.5 & 97.3 & 39& 0 & 0 & 56.3 \\
         44 & 94.5  & 98.2 & 96.5 & 43 & 0 & 0 & 56.2 \\
         52 & 94.9 & 94.7 & 96.4 & 13 & 0 & 0 & 56.2 \\
         42 & 90.6 & 95.8 & 94.7 & 29 & 0 & 0 & 56.1 \\
         30 & 90.9 & 93.0 & 93.7 & 57 & 0 & 0 & 56.1 \\
         10 & 87.0 & 90.7 & 93.6 & 53 & 0 & 0 & 56.1 \\
         14 & 86.1 & 95.5 & 93.2 & 19 & 0 & 0 & 56.0 \\
         54 & 89.6 & 97.3 & 90.3 & 35 & 0 & 0 & 55.9 \\
         22 & 77.8 &89.0  & 89.7 & 5 & 0 & 0 & 55.8 \\
         50 & 49.3 & 87.2 & 89.5 & 37 & 0 & 0 & 37.7 \\
         34 & 85.7 & 88.0 & 89.2 & 55 & 0 & 0 & 37.6 \\
         26 & 84.8. & 80.0 & 82.7 & 11 & 0 & 0 & 37.3 \\
         38 & 72.6 & 82.3 & 82.5 & 7 & 0 & 0 & 36.9  \\
         46 & 71.2 & 82.6 & 82.5 & 3 & 0 & 0 & 25.2 \\
       \bottomrule
    \end{tabular}
    \small
    \caption{\small Predicting the long Collatz steps with different training set distributions. }
    \label{tab:uniform}
\end{table}

Here, we experiment with two training distributions for $k$: uniform ($\mathcal{P}(k)=c$), and log-uniform ($\mathcal{P}(k)=\frac c k$). The distributions of $k'$, and the test distributions are not modified. To sample inputs $n$ from such distributions, we first select a value $l$ from $1$ to $16$, according to a uniform or log-uniform distribution, and sample $u$ uniformly between $1$ and $\lfloor \frac{10^{12}}{2^{l+1}} \rfloor$. Setting $n=u*2^{l+1}+2^l-1$, we must have $k=l$, which follows a uniform or log-uniform distribution over $\{1,\dots,16\}$, and $k'$ that follows the ``regular'' distribution $\mathcal P(k'=l')=\frac {1}{2^{l'}}$. As before, we train models with the same architecture and number of parameters as in our previous experiments, using bases from $2$ to $57$, on training sets of $600$ million random examples. Test sets for these experiments are generated as before, and display the natural (exponentially decreasing) distribution of $k$ and $k'$.

On both training distributions, odd-base models achieve performances close to zero. Even-base models learn, but achieve worse accuracies than in the base experiments when trained from  uniform $k$. On the training sample with log-uniform $k$, they achieve the same performance as the base experiments (Table~\ref{tab:uniform}). Overall, a change in training distribution degrades model performance (or, at best, does not improve it). This is a significant departure from our results on GCD. It suggests that the natural, power-law distribution of $k$ and $k'$ is required (at least in odd-base models) for the learning pattern to emerge.

\section{Discussion}

Many previous works reported that large language model and small transformers struggle on simple arithmetic tasks. Our experiments indicate that, on a complex arithmetic function like the long Collatz step, models using the best encoding bases achieve near-perfect accuracy. Our analysis of model errors also show that even less powerful models perform better than their accuracy results suggest: model failures are explainable because mathematical properties of Collatz sequences have been learned. Still, our results are disappointing: they show that all models, no matter their accuracy, only learn to predict some of their inputs, associated with small values of $k$ and $k'$. This puts a limit on what models can learn, but also tell an interesting story about how transformers learn algorithms.

\paragraph{The limit is the loops, not the arithmetic.} Our analysis of the long Collatz step reveals an important fact about arithmetic transformers: the difficulty of learning a complex arithmetic function like $\kappa(n)$ does not lie in learning the arithmetic calculations (the repeated application of $n \to 3n+1$ and $n \to \frac n 2$), but in learning the control structures of the algorithm: the length of the loops in the computation. In fact, once a pair $(k,k')$ is learned, the model has no difficulty figuring the associated mapping $\kappa_{k,k'}(n)$, and achieving almost perfect accuracy on associated inputs.

This highlights the strangeness of the learning pattern. In most deep learning studies, we think of the model as a universal approximator~\citep{hornik1991approximation}, which begins as a random function, and gets closer to the solution as the training loss is minimized. Since each model update involves a small learning-step-sized move in the general direction of the gradient (computed on a mini-batch and amortized over several optimization steps), the solution learned changes slowly, and almost continuously, over all its domain of definition. This usually causes the accuracy curves be smooth and increasing, and the approximations provided by the model to get closer and closer for all input. 

The learning pattern of the long Collatz step is very different: the function learns by sudden bursts, where accuracy jumps by several percentage points, and when the model learns to predict, perfectly, a whole new class of inputs (associated with loop lengths $k$ and $k'$). We can make this description very precise: the long Collatz successor of $n$, $\kappa(n)$, depends on two functions, $k(n)$ and $k'(n)$, as
$$\kappa(n) = \kappa(n, k(n),k'(n))=\frac {(\frac3 2)^k (n+1) -1}{2^{k'}}.$$

The approximations of $\kappa(n)$ learned by the model belong to a very specific class: restrictions of $\kappa(n,k,k')$ to particular values of $k$ and $k'$. They achieve $100\%$ accuracy on all inputs associated with those values of $k$ and $k'$ and are wrong (but in a principled way) everywhere else. The learning pattern is the sequence of pairs $(k,k')$ learned, it is the same for all bases and model initialization. Somehow, the model does not act as a ``universal approximator'', but is limited to a very small class of functions.

Once a new pair $(k,k')$ is learned, the model learns the corresponding function at once, as if the model already knew the general template $\kappa(n,k,k')$ for all $k$ and $k'$, and only needed to learn to discriminate inputs with $k \geq l$ (for $k'=1$), or $k'\geq l'$ (for $l'>1$), to instantiate a new specialization.

In retrospect, these observations seem applicable to prior work. Models trained to predict the greatest common divisor of two integers learn a classification of their input pairs according to their common divisors (the largest of which becomes the model prediction for this class of input). This is a simpler case, because each input class correspond to one output value (the most common output in the class). Yet, the model learns the solution one class of input at a time, instead of learning better and better approximations, over the full domain of definition. We believe a similar patter occurs for length generalization: models learning to predict input of a different length that those seen during training.

\paragraph{Classifying inputs one bit at a time.}

Our observations from Sections~\ref{sec:theory} and~\ref{sec:hierarchy} demonstrate that learning the loops amounts to classifying inputs by their binary suffixes. The learning pattern indicates that this is done one bit at a time.

In odd-base models, all odd inputs (ending in $1_2$) are initially predicted as $\kappa_{1,1}(n)$, for an initial accuracy of $25\%$. After some time, the next accuracy level ($37.5\%$) is learned, and input are split into two classes:
\begin{itemize}[nosep] 
    \item inputs ending in $11_2$ ($k \geq 2$) predicted as $\kappa_{2,1}(n)$,
    \item inputs ending in $01_2$ ($k=k'=1$) predicted (as before) as $\kappa_{1,1}(n)$.
\end{itemize}
The next step brings accuracy to $55\%$, with the model predicting:
\begin{itemize}[nosep] 
    \item inputs ending in $111_2$ ($k \geq 3$) as $\kappa_{3,1}(n)$,
    \item inputs ending in $011_2$ ($k \geq 3$) as $\kappa_{2,1}(n)$,
    \item inputs ending in $101_2$ ($k \geq 3$) as $\kappa_{1,2}(n)$,
    \item inputs ending in $001_2$ ($k=k'=1$) as $\kappa_{1,1}(n)$.
\end{itemize}
Each input class splits in two, as the model learns the third bit of the binary representation.

The next step (to $71\%$ accuracy), involves three new splits (that can be observed as intermediary steps in some learning curves from Figure~\ref{fig:stepcurves}): 
\begin{itemize}[nosep] 
    \item $111_2$ into $1111_2$, predicted as $\kappa_{4,1}(n)$ and $0111_2$, predicted as $\kappa_{3,1}(n)$,
    \item $011_2$ into $0011_2$, predicted as $\kappa_{2,2}(n)$ and $1011_2$, predicted as $\kappa_{2,1}(n)$,
    \item $101_2$ into $0101_2$, predicted as $\kappa_{1,3}(n)$ and $1101_2$, predicted as $\kappa_{1,2}(n)$.
\end{itemize}

Each step in the learning pattern corresponds to the model learning one more bit in the representation of its inputs. This is easier when the input is encoder in an even base. 

How odd-base models learn the binary representation of their input remains a mystery. The classical algorithm for computing the residual modulo $2^p$ of an  integer represented in base $B$ involves computing a weighted sum of the digits (by the powers of $B$) modulo $2^p$, but prior work has shown that such tasks are hard for transformers~\citep{saxena2025}.

\paragraph{On the absence of hallucinations.} On average, $97\%$ of model errors are accounted for by our analysis of model failures (Section~\ref{sec:errors}). Up to some small rounding error, model predictions are close to $\kappa_{l,l'}(n)$, for particular (and predictable) values of $l$ and $l'$. Our models never hallucinate random solutions. This observation was made in previous works~\citep{charton2024gcd, charton2022,butter2024extrapolatingjetradiationautoregressive}, but it contradicts observations about large language models (LLM), which frequently hallucinate. Hallucinations happen when plausibility is not well estimated (because of lack of data), or when it is not a good proxy for truth. In mathematics, truth is decided by specific rules. Plausibility may work for oft repeated statements, like $2+2=4$, but it is not a good proxy for truth. Supervised models, on the other hand, are trained on true statements. Models still make errors, but they do not predict at random.

\paragraph{Can transformers learn the long Collatz step?} While some of our models achieve surprisingly high accuracies, our observations about the learning pattern bring a disappointing conclusion. The models do not learn a general algorithm, valid for all inputs, but a near-perfect approximation of $\kappa(n)$ for a large class of inputs. As long as training data is finite, some inputs will not be covered. We believe this is a general fact about deep learning models. 
Note, however, that this is not specific to deep learning. Most algorithms implemented on a computer present such limitations, e.g. they are restricted to $32$, or $64$-bit integers, or limited to floating point arithmetic.

\paragraph{A new path to explainability.}
Our approach to model interpretability differs from most previous works. Instead of investigating model parameters and learned representations, we treat models as deterministic black boxes, and experiment on them by comparing their predictions for specific inputs. This approach is well-suited to problems of mathematics, where known theoretical properties of the problem can be leveraged to help understand model predictions. For instance, the analysis from Section~\ref{sec:pattern} point to a possible link between model prediction and the binary representation of inputs, that we can associate to loop length by the theoretical arguments of Section~\ref{sec:theory}. In Section~\ref{sec:errors}, the empirical distribution of the ratios $p/t$ and the values of their modes ($2$ and $\frac 2 3$) suggest a hypothesis about model predictions that could be verified on test examples, and generalized into a full hierarchy of errors. 
We believe this approach is worth developing, both as a tool for understanding deep learning and transformers and for mathematical discovery.

\paragraph{Mathematics as a tool for understanding transformers.} Modern deep learning architectures, like transformers, are difficult to understand, for two reasons. 
First, they are often evaluated on annotated datasets, of variable quality, and with binary metrics (model answers are either right or wrong), that make failure analysis almost impossible. Second, the distribution of training and test data are difficult to control. This raises concerns about contamination and the values of claims on model generalization.

Problems of mathematics offer a solid alternative for benchmarking and interpreting models. They come with an objective and verifiable criterion for success, and failures can usually be analyzed (as we do here) by leveraging theoretical knowledge about the problem. The underlying theory can also help explain model predictions, and design experiments, as we did in Section~\ref{sec:ablations} by replacing the task with an easier one. 
Math datasets are usually generated. This provides control over the training and test distributions, and helps prevent data contamination.

\paragraph{Learning about Collatz, transformers for mathematical discovery.} Experimenting with transformers on problems of mathematics can also be a tool for mathematical discovery. Our initial experiments and the analysis of the learning pattern demonstrate the importance of the  binary representation of $n$, and suggest its relation with $k$ and $k'$, a non-obvious fact in the case of $k'$ that we could prove. More generally, experimenting with different representations of a problem (e.g. different bases for tokenization), or the presence of additional features in the training set (e.g. learning $k$ and $k'$ instead of $\kappa(n)$), and observing how this helps or hinders, learning, can provide insight into the problem, and suggest new directions for research. We believe this constitutes a potential application of artificial intelligence to mathematics and science.

The availability of source code and compute makes this approach quite affordable. All experiments in this paper run in a few hours on one machine with an obsolete GPU (V100), and relies on open-source codebase (Int2Int\footnote{https:/www.github.com/f-charton/Int2Int}).

\paragraph{Acknowledgements}
The initial ideas for this research came up during the AI and Maths program organized by the Center for Mathematical Sciences and Applications in Harvard, from September to November 2024. We thank Dan Freed and Mike Douglas for hosting us. This research was initiated while Fran\c{c}ois worked at Meta. Ashvni thanks the Hooper-Shaw foundation and the Sydney Mathematical Research Institute for supporting her.

%\newpage 
\bibliography{padic}

\end{document}